%% file: arxiv_main.tex
\documentclass[11pt]{article}

\input{macros}

\begin{document}

\title{Contextual Dynamic Pricing with Heterogeneous Buyers}

\author[1]{Thodoris Lykouris}
\author[2]{Sloan Nietert}
\author[3]{Princewill Okoroafor}
\author[1]{\authorcr Chara Podimata} %
\author[4]{Julian Zimmert}
\affil[1]{MIT, \texttt{\{lykouris, podimata\}@mit.edu}}
\affil[2]{EPFL, \texttt{sloan.nietert@epfl.ch}}
\affil[3]{Harvard, \texttt{pco9@cornell.edu}}
\affil[4]{Google, \texttt{zimmert@google.com}}

\date{\today}

\maketitle

\begin{abstract}
We initiate the study of contextual dynamic pricing with a heterogeneous population of buyers, where a seller repeatedly posts prices (over $T$ rounds) that depend on the observable $d$-dimensional context and receives binary purchase feedback. Unlike prior work assuming homogeneous buyer types, in our setting the buyer's valuation type is drawn from an unknown distribution with finite support size $K_{\star}$. We develop a contextual pricing algorithm based on optimistic posterior sampling with regret $\widetilde{O}(K_{\star}\sqrt{dT})$, which we prove to be tight in $d$ and $T$ up to logarithmic terms. Finally, we refine our analysis for the non-contextual pricing case, proposing a variance-aware zooming algorithm that achieves the optimal dependence on $K_{\star}$.
\end{abstract}

\input{sections/1-intro}
\input{sections/2-prelims}
\input{sections/3-alg}
\input{sections/4-zooming}
\input{sections/5-extras}
\input{sections/6-discussion}

\input{sections/7-acknowledgements}

\bibliographystyle{apalike}
\bibliography{references}

\newpage
\begin{center}
\begin{Large}\textbf{Appendix}
\end{Large}
\end{center}

\appendix
\input{sections/A-prelims-proofs}
\input{sections/B1-alg-proofs}
\input{sections/B2-lb-proofs}
\input{sections/C-zooming-proofs}

\input{sections/D-extras-proofs}

\end{document}

%% file: macros.tex
\usepackage[utf8]{inputenc}
\usepackage[T1]{fontenc}
\usepackage{microtype} %
\usepackage[margin=1in]{geometry}

\usepackage{natbib}

\usepackage{xparse}
\usepackage{xspace} 
\usepackage{booktabs}
\usepackage{sidecap} 

\usepackage{amsmath, amsthm, amssymb}
\usepackage{mathtools}
\usepackage{thmtools}
\usepackage{authblk}

\usepackage[shortlabels]{enumitem}

\usepackage{algorithm}
\usepackage[noend]{algpseudocode}
\usepackage[algo2e,noend,ruled,linesnumbered]{algorithm2e} %

\usepackage{dsfont}
\usepackage{multirow}

\usepackage[colorlinks=true]{hyperref}
\hypersetup{
citecolor = blue,
}
\usepackage[capitalise]{cleveref}

\usepackage{color-edits}

\usepackage{import}

\newcommand{\xhdr}[1]{\vspace{1mm} \noindent{\textbf{#1}}}
\newcommand{\squishlist}{
 \begin{list}{$\bullet$}
  { \setlength{\itemsep}{0pt}
     \setlength{\parsep}{3pt}
     \setlength{\topsep}{3pt}
     \setlength{\partopsep}{0pt}
     \setlength{\leftmargin}{1.5em}
     \setlength{\labelwidth}{1em}
     \setlength{\labelsep}{0.5em} } }

\newcommand{\squishend}{
  \end{list}  }

\theoremstyle{plain}
\newtheorem{lemma}{Lemma}[section]
\newtheorem{theorem}[lemma]{Theorem}

\newtheorem{remark}[lemma]{Remark}

\theoremstyle{definition}
\newtheorem{definition}[lemma]{Definition}
\newtheorem{assumption}[lemma]{Assumption}
\newtheorem{setting}[lemma]{Setting}

\newtheorem{fact}[lemma]{Fact}

\makeatletter
\renewcommand{\ALG@beginalgorithmic}{\small}
\makeatother

\definecolor{pink}{RGB}{255,20,147}
\definecolor{PINK}{RGB}{255,20,147}
\definecolor{darkgreen}{RGB}{1,150,32}
\addauthor{tl}{cyan}
\addauthor{cp}{red}
\addauthor{sn}{blue}
\addauthor{po}{pink}
\addauthor{pco}{pink}
\addauthor{jz}{green}

\renewcommand{\epsilon}{\varepsilon}

\DeclareMathOperator*{\E}{\mathbb{E}}

\DeclareMathOperator*{\argmax}{arg\,max}

\DeclareMathOperator{\Unif}{Unif}

\DeclareMathOperator{\supp}{supp}
\DeclareMathOperator{\sign}{sgn}

\DeclareMathOperator{\Ber}{Ber}
\DeclareMathOperator{\poly}{poly}
\DeclareMathOperator{\KL}{KL}

\newcommand{\dem}{\mathsf{dem}}
\newcommand{\rev}{\mathsf{rev}}
\newcommand{\br}{\mathsf{br}}
\newcommand{\gap}{\mathsf{gap}}

\newcommand{\dis}{\mathsf{dis}}

\newcommand{\dL}{d_{\mathrm{L}}}
\newcommand{\proj}{\mathsf{proj}}
\newcommand{\REG}{\mathrm{REG}}

\newcommand{\LS}{\mathrm{LS}}
\newcommand{\FG}{\mathrm{FG}}

\newcommand{\TV}{\mathrm{TV}}
\newcommand{\price}{\mathsf{price}}

\newcommand{\OPS}{\textup{\texttt{OPS}}\xspace}
\newcommand{\POPS}{\textup{\texttt{POPS}}\xspace}
\newcommand{\GOPS}{\textup{\texttt{GOPS}}\xspace}
\newcommand{\Zooming}{\textup{\texttt{ZoomV}}\xspace}
\newcommand{\ZoomDim}{\textup{\texttt{ZoomDim}}\xspace}
\newcommand{\ZoomDimV}{\textup{\texttt{ZoomDimV}}\xspace}

\newcommand{\ProjectedVolume}{\textup{\texttt{ProjectedVolume}}\xspace}

\newcommand{\R}{\mathbb{R}}

\newcommand{\eps}{\varepsilon}
\newcommand{\dd}{\mathrm{d}}
\newcommand{\PP}{\mathbb{P}}

\newcommand{\cA}{\mathcal{A}}
\newcommand{\cB}{\mathcal{B}}
\newcommand{\cC}{\mathcal{C}}
\newcommand{\cD}{\mathcal{D}}
\newcommand{\cE}{\mathcal{E}}
\newcommand{\cF}{\mathcal{F}}
\newcommand{\cG}{\mathcal{G}}

\newcommand{\cI}{\mathcal{I}}

\newcommand{\cP}{\mathcal{P}}

\newcommand{\cS}{\mathcal{S}}
\newcommand{\cT}{\mathcal{T}}
\newcommand{\cU}{\mathcal{U}}

\newcommand{\cX}{\mathcal{X}}

\newcommand{\cZ}{\mathcal{Z}}

\newcommand{\1}{\mathds{1}}
\newcommand{\unitball}{\mathbb{B}^d}
\newcommand{\unitsph}{\mathbb{S}^{d-1}}

\newcommand{\idx}{\mathsf{index}}
\newcommand{\UCB}{\mathsf{UCB}}

\newcommand{\defeq}{\coloneqq}

\allowdisplaybreaks

\newcommand{\ubrace}[2]{{\underbrace{\ifsqrt\expandafter\cramped\fi{#1}}_{#2}}}

\newif\ifsqrt
\newsavebox{\usqrtbox}

\makeatletter
\newcommand{\raisemath}[1]{\mathpalette{\raisem@th{#1}}}
\newcommand{\raisem@th}[3]{\raisebox{#1}{$#2#3$}}
\makeatother

\usepackage{mleftright}

\def\left#1{\mleft#1}
\def\right#1{\mright#1}

\usepackage{mparhack}
\usepackage{ifoddpage}
\usepackage{marginnote}
\newcommand{\prf}[1]{%
\reversemarginpar
\marginnote{\mbox{\hyperref[prf:#1]{\color{blue!30!white}{\normalfont\texttt{[proof]}}}}\hfill}%
}

\newenvironment{altproof}[1]%
{%
\par\noindent{\bfseries\upshape#1\ }%
}
{\qed}

%% file: sections/1-intro.tex
\section{Introduction}
\label{sec:intro}

In online learning for contextual pricing, a learner ({aka} seller) repeatedly {sets} prices {for} different products {with the goal of maximizing revenue through interactions with agents (aka buyers or customers)}. {Concretely,} in {each round} $t=1,\ldots, T$, {nature selects a} product {with} a $d$-dimensional feature representation $u_t$ (context) and the seller selects a price~$p_t \geq 0$. In the simplest variant, the \emph{linear valuation model}, customers have a fixed intrinsic valuation model (type) that is unknown to the learner; this has a $d$-dimensional representation $\theta^{\star}$ whose coordinates reflect the valuation that each product feature adds, i.e., the customer's valuation is $v_t=\langle \theta^{\star},u_t\rangle+\epsilon_t$ where $\epsilon_t$ is a noise term. The customer makes a purchase only when their valuation is higher than the price, i.e., $v_t\geq p_t$. The learner's goal is to maximize revenue, i.e., the sum of the prices in rounds when purchases occur. An equivalent {objective} is to minimize \emph{regret}, which  {is measured} against a benchmark that always selects the customer's valuation as the price {for the given round}. 

Before outlining our contributions, we highlight the unique challenge of online learning in contextual pricing. 
One key difficulty is that the learner faces both an infinite action space (i.e., all possible prices) and a discontinuous revenue function; 
indeed, even small price increases can deter a buyer from purchasing, hence causing sharp revenue loss for the learner. However, the problem offers a richer feedback structure than classical multi-armed bandits: a non-purchase indicates that all higher prices would also be rejected by the buyer, while a purchase confirms that all lower prices would be accepted too. The two primary approaches from the literature to tackle this problem involve estimating the unknown parameter $\theta^{\star}$ through online regression or multi-dimensional binary search (see Section~\ref{sec:rel-work} for further discussion).

A crucial limitation for both approaches is that they require all customers to behave \emph{homogeneously} according to a single type $\theta^{\star}$; see the related work section for results robust to small deviations from this assumption. Moving beyond this homogeneity assumption, we pose the following question:
\begin{center}
      \emph{How can one design contextual pricing algorithms with a heterogeneous population of customers?} 
\end{center}

\subsection{Our Contribution}

\xhdr{Our setting (Section~\ref{sec:prelims}).} 
To study contextual pricing with a heterogeneous buyer population, we assume that the type $\theta_t$ in round $t$ is drawn from a fixed, unknown distribution $D_{\star}$. When $D_\star$ is supported on a single type $\theta_\star$, we recover the homogeneous setting. In our setup, the number of distinct buyer types $K_{\star} = |\supp{D_\star}|$ reflects the \emph{degree of heterogeneity}. We assume that $K_{\star} > 1$ throughout the paper.

There are several obstacles to applying existing algorithms from the literature.
First, canonical contextual pricing algorithms based on regression either compete against (simple) linear policies or assume context-independent and identically distributed (i.i.d.) valuation noise. In contrast, the optimal policy in our setting may best respond based on a \emph{context-dependent} type rather than a \emph{fixed} type, and the stochasticity due to heterogeneity is inherently context-dependent and thus non-i.i.d.
Second, given that \emph{the buyer types are not observable}, one cannot connect the observed feedback to shrinkage of type-dependent uncertainty sets; this rules out running canonical multi-dimensional binary search / contextual pricing algorithms for each buyer type in parallel. Third, since in our setting there is a \emph{continuum of actions}, any canonical contextual bandits algorithm whose regret scales with the discretized action count (e.g., EXP4) will suffer suboptimal performance.

\xhdr{Our contextual pricing algorithm (Section~\ref{sec:alg}).} To tackle the above challenges, we employ recent advances in the contextual bandit literature that attain a better scaling with the number of actions, thus evading the shortcomings of EXP4 with na\"ive discretization. In particular, we build on the \emph{optimistic posterior sampling} (OPS) approach \citep{zhang2022thompson} which, in our setting, maintains a posterior $\mu_t$ over all candidate type distributions. We call these candidate type distributions \emph{models} and refer to their (possibly infinite) family as $\cD$. At a high level, in every round, OPS best responds to a model sampled from $\mu_t$. As typical in online learning, the posterior update penalizes models that \emph{disagree} with the observed feedback (\emph{model mismatch}) aiming to converge to the model $D_{\star}$. To encourage exploration in the absence of full information, this penalty is reduced by an \emph{optimism bias} term that rewards models with the highest potential to positively contribute to the revenue. The OPS approach enables regret bounds of $\sqrt{T\cdot c \cdot \log|\cD|}$, scaling with a \emph{disagreement coefficient} $c$ that measures the per-context structural complexity of the reward functions and captures the tension between exploration and exploitation.
This coefficient is always bounded by the number of actions but can be much smaller in general.

Our main technical contributions in adapting OPS to heterogeneous contextual pricing are twofold. First, to bound the disagreement coefficient $c$, we observe that, for any fixed context, the aggregate demand function induced by~$D_{\star}$ has at most~$K_{\star}$ ``jumps'',\footnote{Each ``jump'' corresponds to a change of type from (say) type $i$ to type $i+1$.} thus creating $K_{\star}+1$ intervals. Over each interval, we bound the disagreement coefficient by a factor of $2$. Combining these arguments with a novel decomposition lemma for the disagreement coefficient of functions with $K_\star$ breakpoints, we show that $c\leq 2(K_{\star}+1)$. When $K_{\star}$ is known, we apply a variant of OPS over a finite covering of the class $\mathcal{D}$ containing \emph{all} possible distributions over $K_\star$ types, of log cardinality $dK_\star \log T$. 
Second, to extend our sublinear regret guarantee to the infinite model class $\mathcal{D}$, we modify OPS to conservatively perturb its recommended prices (which cannot overly impact regret due to one-sided Lipschitzness of the revenue function). We then construct a coupling between the actual trajectory of OPS and one where $D_\star$ belongs to the finite cover, allowing us to transfer regret bounds.
Finally, we adapt to unknown $K_\star$ by initializing OPS with a non-uniform prior over models.
These technical contributions enable us to show a regret guarantee of $\widetilde{O}(K_\star\sqrt{dT})$. Finally, we show that this guarantee is optimal (up to logarithmic terms) with respect to the dependence on both the contextual dimension $d$ and the time horizon $T$, establishing a lower bound of $\Omega(\sqrt{K_\star d T})$ for sufficiently large $T = \Omega(dK_\star^3)$.

\xhdr{Non-contextual improvements (Section~\ref{sec:zooming}).} The above upper and lower bounds raise a natural question on the optimal dependence of the regret on the number of buyer types $K_{\star}$; we resolve this question in the non-contextual version of the problem ($d=1$) by providing an algorithm with an upper bound of $\widetilde{O}(\sqrt{K_\star T})$. Our algorithm, \Zooming, combines zooming (i.e., adaptive discretization) methods from Lipschitz bandits \citep{slivkins2008} with variance-aware confidence intervals \citep{audibert2009exploration}. Our analysis shows that the regret of \Zooming scales with a novel variance-aware zooming dimension that can be significantly smaller than the standard measure of complexity for Lipschitz bandits. For pricing, this variance adaptation unlocks our $\widetilde{O}(\min\{\sqrt{K_\star T}, T^{2/3}\})$ bound (versus $O(T^{2/3})$, obtained via the standard zooming analysis).

The non-contextual version of pricing for heterogeneous buyers was previously studied by \cite{cesabianchi19pricing}, who establish a matching upper and lower bound if all types are ``well-separated'' from each other. In independent and concurrent work, \cite{bacchiocchi25piecewise} remove this assumption and also achieve a regret bound of $\widetilde{O}(\sqrt{K_\star T})$. Both of these algorithms employ binary search techniques that are specialized to the piecewise-linear revenue function. In contrast, \Zooming is a generic (one-sided) Lipschitz bandits algorithm, and the structure of the revenue function is only used within its analysis, to bound the variance-aware zooming dimension. One feature of this approach is that we also achieve $T^{2/3}$ regret when $K_\star \gg T^{1/3}$, without changing the~algorithm.

\xhdr{Stronger type observability (Section~\ref{sec:extras}).} Finally, we consider contextual pricing under the assumption that the learner can \emph{identify} each arriving type, i.e., where the learner observes ex-post information about the sampled type $\theta_t$. 
We analyze two observability models: one where the learner receives a discrete identifier $z_t \in [K_\star]$ --- under which a computationally efficient pricing algorithm matches the $\widetilde O(K_\star \sqrt{dT })$ regret bound of OPS --- and another where the full type vector $\theta_t \in \unitball$ is observed --- for which we reduce the dependence on $K_\star$ and $d$ to achieve regret $\widetilde{O}(\sqrt{\min\{K_\star, d\}T})$. 
These results demonstrate how richer feedback reduces complexity in dynamic pricing.

\subsection{Related Work}
\label{sec:rel-work}

Our work relates closely to three lines of work: (I) contextual pricing/search; (II) non-contextual pricing; and (III) Lipschitz bandits.

\xhdr{(I) Contextual pricing/search.} The closest line to our work is \emph{contextual pricing/search}. In contextual search, there is a repeated interaction between a learner and nature, where the learner is trying to learn a hidden vector $\theta^{\star} \in \mathbb{R}^d$ over time while receiving only single-bit feedback. Mathematically, at each round $t \in [T]$, the learner receives a (potentially adversarially chosen) context $u_t \in \mathbb{R}^d$ and decides to query $y_t \in \mathbb{R}$. The learner receives feedback $\sigma_t = \sign(\langle u_t, \theta^{\star} \rangle - y_t) \in \{-1, +1\}$ and incurs loss $\ell_t(y_t, \langle u_t, \theta^{\star}\rangle)$ \citep{cohen20pricing}. Notably, the learner does \emph{not} observe $\ell_t(y_t, \langle u_t, \theta^{\star}\rangle)$; only the binary feedback $\sigma_t$. 
When the loss function $\ell_t$ corresponds to the lost revenue as a result of posting price $y_t$ (i.e., the ``pricing loss''), this setting reduces to ours with a \emph{homogeneous} buyer population, i.e., $K_\star = 1$.
The contextual search literature has also considered two other loss functions: the symmetric/absolute loss $\ell_t(y_t, \langle u_t, \theta^{\star} \rangle) := |y_t - \langle u_t, \theta^{\star} \rangle|$ and the $\eps$-ball loss $\ell_t(y_t, \langle u_t, \theta^{\star} \rangle) = \1 \left\{ |y_t - \langle u_t, \theta^{\star} \rangle| > \eps\right\}$, which are motivated by settings other than pricing.

There have been two approaches in the literature for learning in contextual pricing and contextual search (for homogeneous agents/buyers). The first approach (e.g., \citealp{cohen20pricing,lobel18contextual,leme22contextual, liu21contextual}) employs a version of multidimensional binary search: specifically, the algorithms maintain a ``knowledge set'' with all the possible values of $\theta^{\star}$ which are ``consistent'' with the feedback that nature has given thus far. Similar to traditional binary search, the query point is chosen to be the point that (given the nature's feedback) will eliminate roughly half of the current knowledge set. As the knowledge set shrinks, the learner ends up with a small knowledge set for the possible values of $\theta^{\star}$; this is enough to guarantee sublinear regret. The series of works in~\citep{cohen20pricing,lobel18contextual,leme22contextual, liu21contextual} optimized regret bounds for the three different loss functions (i.e., symmetric, $\eps$-ball, and pricing). The specific algorithms were different at each paper, but they all maintained a ``binary search'' flavor. Most of the algorithms employing a multidimensional binary search approach can be ``robustified'' to very little noise in the agents' responses; since the learner will irrevocably shrink the knowledge set according to the feedback received from nature, they can only afford very few~mistakes. 

The second approach (e.g., \citealp{javanmard2019dynamic, javanmard2017perishability, fan2024policy, luo2024distribution}) focuses exclusively on pricing settings. This approach uses regression-based algorithms for learning the correct price and \emph{needs to assume} stochastic noise in the buyers' responses. There have also been works studying other aspects of contextual pricing (e.g., strategic agents \citep{amin2014repeated} and unknown noise distribution \citep{xu2022towards}). Apart from the methodological differences with our work, both streams of literature focus on a \emph{homogeneous} agent population and cannot be readily adapted for a \emph{heterogeneous} population setting.

Moving closer to the heterogeneous agents problem, \cite{krishnamurthy23contextual} studied ``corruption-robust'' contextual search, where the agent population is mostly homogeneous, except for $C = o(T)$ corrupted agent responses. Their regret bounds were subsequently strengthened by \cite{leme22corruptcontextual}, but the latter approach only works for contextual search with absolute and $\epsilon$-ball loss and does \emph{not} cover the pricing loss. This model has been also studied with a Lipschitz target function \citep{zuo2024corruption}. Learning with corruptions can be seen as a first step towards learning from heterogeneous agents, but the approaches above do not scale appropriately for truly heterogeneous agent populations. In contrast, we focus on \emph{fully heterogeneous} settings, where we do not constrain the number or the size of the different buyer types.

\xhdr{(II) Non-contextual pricing.} The special case of $d = 1$, where there is no context to inform decision-making, was introduced by~\cite{kleinberg03demand}, who studied non-contextual dynamic pricing for a homogeneous, stochastic, and adversarial buyer population. For the adversarial buyer population, the authors assumed that there can be $T$ different valuations and showed tight regret bounds of $\widetilde{O}(T^{2/3})$. In contrast, in our setting, we assume that users are ``clustered'' in $K_{\star}$ types), and so the lower bound of $\Omega(T^{2/3})$ of \citet{kleinberg03demand} does not apply. 

The closest to our work is the work of \cite{cesabianchi19pricing}, who consider pricing a heterogeneous agent population with an unknown number of types, but the types are still limited to be less than $o(T)$. Throughout the paper, we discuss how their bounds relate to ours for the special case of $d = 1$. None of the aforementioned techniques readily generalize to contextual pricing~settings.

\xhdr{(III) Lipschitz bandits.} Our work is also related to the literature on Lipschitz bandits. Although the pricing loss is \emph{not} fully Lipschitz, it has recently been observed that it satisfies a \emph{one-sided Lipschitzness}. This allows us to leverage techniques from adaptive discretization \citep{slivkins2008} to obtain improved bounds for $d = 1$.
Zooming had previously been applied to pricing (see, e.g., \citealp{podimata2021adaptive}), but these algorithms are insufficient for the $K_\star$-types setting. Indeed, their performance scales with a zooming dimension $\ZoomDim$ that is too large here. On the other hand, \Zooming uses variance-aware confidence intervals so that its performance scales with a smaller, variance-aware zooming dimension $\ZoomDimV$. In particular, while \ZoomDim can equal 1 for worst-case instances, we show that \ZoomDimV is 0 (with a lower-order scaling constant at most $K_\star$).

Finally, \cite{krishnamurthy2020contextual} consider contextual bandits with continuous action spaces, which encompasses the setting of this work. Their regret bounds cover the case where, for a fixed context, the expected reward is Lipschitz in the learner's action. Although their analysis can be adapted to the one-sided Lipschitz setting of pricing, their results either require stochastic contexts or incur large regret due to na\"ive discretization. Even in the stochastic case, their regret bound scales with a policy zooming coefficient that does not appear to admit a useful bound in terms~of~$K_\star$.

%% file: sections/2-prelims.tex
\section{Setup and Preliminaries}
\label{sec:prelims}

\xhdr{Notation.} 
Let $\|\cdot\|$ and $\langle \cdot,\cdot \rangle$ denote the Euclidean norm and inner product on $\R^d$. Let $\unitsph, \unitball \subseteq \R^d$ denote the unit sphere and ball, respectively. Let $\Delta(S)$ denote the set of all probability measures on a measurable set $S \subseteq \R^d$, and let $\supp(D)$ denote the support of $D \in \Delta(\R^d)$. We use $\Delta_k(S)$ for those $D \in \Delta(S)$ with $|\supp(D)| \leq k$. For a positive integer $m$, let $[m] \defeq \{1,2,\dots,m\}$. 

\xhdr{Problem setup.}
We consider $T$ rounds of repeated interaction between a seller, a population of buyers, and an adversary. At each round $t \in [T]$, the seller posts a price $p_t \in [0,1]$ for an item to be sold and a buyer, sampled from the population, decides whether or not to buy the item based on their valuation $v_t \in [0,1]$. We denote the indicator of their purchase by $y_t = \mathds{1}\{v_t \geq p_t\}$. The valuation of the buyer is determined by two factors: their \emph{type} $\theta_t$, which encodes their intrinsic preferences, and an external \emph{context} $u_t$, which describes the current item to be sold and any relevant environmental factors. {The learner does \emph{not} know $\theta_t$, but they do know $u_t$.} We employ a linear valuation model, supposing that $\theta_t$ and $u_t$ lie in $d$-dimensional spaces $\Theta \subseteq [0,1]^d$ and $\cU \subseteq \unitsph$, respectively, and take $v_t = \langle \theta_t,u_t \rangle$. We assume that $\langle \theta,u \rangle \in [0,1]$ for all $\theta \in \Theta$ and $u \in \cU$. We impose no further assumptions on the contexts, allowing them to be generated (potentially adaptively) by the adversary. On the other hand, we assume that each $\theta_t$ is sampled independently from a fixed distribution $D_\star \!\in\! \Delta(\Theta)$ that describes the buyer population, unknown to the seller. All together, the following occur at each round $t \in [T]$:
\squishlist
    \item[1.] the adversary selects a context $u_t \in \cU$;
    \item[2.] a buyer arrives with type $\theta_t \in \Theta$ sampled independently from $D_\star$, with valuation $v_t = \langle u_t, \theta_t \rangle$;
    \item[3.] the seller observes $u_t$ and posts price $p_t \in [0,1]$ for the item;
    \item[4.] the seller observes the purchase decision $y_t = \mathds{1}\{v_t \geq p_t\}$ and receives revenue $p_t y_t$.
\squishend

\xhdr{Benchmark.}
The seller's goal is to maximize their cumulative revenue compared to that which they could have achieved with knowledge of $D_\star$. To express this concisely, we introduce some additional notation. Each distribution $Q$ over valuations in $[0,1]$ induces the following:
\squishlist
    \item a demand function $\dem_Q(p) \defeq \mathbb{P}_{v \sim Q}[v \geq p]$,
    \item an expected revenue function $\rev_Q(p) \defeq p \cdot \dem_Q(p)$,
    \item a revenue-maximizing best response $\br_Q \defeq \argmax_{p \in [0,1]} \rev_Q(p)$ (breaking ties arbitrarily),
    \item and a gap function $\gap_Q(p) \defeq \rev_Q(\br_Q) - \rev_Q(p)$.
\squishend

Once we restrict to a fixed context $u \in \cU$, each type $\theta \in \Theta$ induces valuation $v = \langle u, \theta \rangle$. Thus, each type distribution $D \in \Delta(\Theta)$ induces a projected valuation distribution $Q = \proj(D,u) \in \Delta([0,1])$, defined as the law of $\langle u, \theta \rangle$ when $\theta \sim D$. We then set $\dem_D(p,u) \defeq \dem_{Q}(p)$, $\rev_D(p,u) \defeq \rev_{Q}(p)$, $\br_D(u) \defeq \br_{Q}$, and $\gap_D(p,u) \defeq \gap_{Q}(p)$, accordingly. We abbreviate a subscript of $D_\star$ by ``$\star$'' alone, writing, e.g., $\dem_\star(p,u)$ and $\br_\star(u)$. These details are summarized in \cref{tab:notation}.\smallskip

A seller policy $\cA$ is a (potentially randomized) map from a history $\{u_\tau,p_\tau,y_\tau\}_{\tau = 1}^{t-1}$ and the current context $u_t$ to a posted price $p_t$. An adversary policy $\cB$ is a (potentially randomized) map from a history $\{u_\tau,\theta_\tau,p_\tau,y_\tau\}_{\tau \in [t-1]}$ to the next context $u_t$. We then define the seller's \emph{pricing regret} by
\begin{center}$R_{\cA,\cB}(T) = \sum_{t \in [T]} \gap_\star(p_t,u_t) = \sum_{t \in [T]} \bigl(\rev_\star(\br_\star(u_t),u_t) - \rev_\star(p_t,u_t)\bigr),$ \end{center}
where $\{u_t,p_t\}_{t \in [T]}$ are selected according to $\cA$ and $\cB$.
We will omit the policies from the subscript when clear from context. We focus on controlling the pricing regret in expectation, and will say that $\cA$ satisfies a regret bound $f(T)$ if $\E[R_{\cA,\cB}(T)] \leq f(T)$ for all $\cB$.\smallskip

Our guarantees will scale with context dimension $d$ and the \emph{degree of heterogeneity}, which we quantify via the support size $K_\star \defeq |\supp(D_\star)|$ (that may be infinite). We do \emph{not} assume that $K_\star$ is known to the seller. Designing an effective seller policy is challenging because $D_\star$, $K_\star$, and the realized buyer types are unknown to the seller, who must carefully balance exploration and exploitation given only the current context and the history of purchase outcomes. 

\begin{table}[b]
    \begin{tabular}{ll@{\hskip 1.8cm}ll}
        \toprule
        \multicolumn{2}{c}{\textbf{Problem parameters (known)}} 
        & \multicolumn{2}{c}{\textbf{Instance parameters (unknown)}} \\
        \midrule
        context dimension & $d$
        & type distribution & $D_\star \in \Delta(\Theta)$ \hspace*{2.1cm}\\[2pt]
        context/feature space & $\cU \subseteq \unitsph$
        & \# types & $K_\star = |\supp(D_\star)|$ \\[2pt]
        type/preference space & $\Theta \subseteq [0,1]^d$
        &  &  \\
    \end{tabular}
    \vspace{2mm}

    \begin{tabular}{llll}
        \toprule 
        \multicolumn{2}{c}{\textbf{Non-contextual primitives}} 
        & \multicolumn{2}{c}{\textbf{Contextual primitives}} \\[2pt]
        \multicolumn{2}{c}{w.r.t.\ valuation dist.\ $Q \in \Delta([0,1])$} 
        & \multicolumn{2}{c}{w.r.t.\ type dist.\ $D \in \Delta(\Theta)$ and context $u \in \cU$} \\
        \midrule
        demand & $\dem_Q(p) = \PP_{v \sim Q}[v \ge p]$
        & projected value dist.\ & $\proj(D,u) \in \Delta([0,1])$ \\[2pt]
        revenue & $\rev_Q(p) = p \cdot \dem_Q(p)$
        & rev./dem./gap & $f_D(p,u) \!=\! f_{\proj(D,u)}(p)$ \\[2pt]
        gap & $\gap_Q(p) \!=\! \max\limits_x \rev_Q(x) \!-\! \rev_Q(p)$
        & rev./dem./gap for $D_\star$ & $f_\star(p,u) \!=\! f_{D_\star}(p,u)$ \\
    \end{tabular}
    \vspace{2mm}
    
    \caption{Summary of main notation.}
    \label{tab:notation}
\end{table}

\xhdr{Basic pricing facts.} Finally, we provide some basic properties of the pricing problem, with proofs in~\cref{app:prelims-proofs}. Essential for this work is the one-sided Lipschitzness of the expected revenue function. This is a consequence of the monotonicity of demand functions, and it has previously been used to apply techniques from Lipschitz bandits to non-contextual pricing \citep{podimata2021adaptive}.

\begin{lemma}[One-sided Lipschitzness]
\label{lem:one-sided-lipschitzness}
Fix any distribution $Q \in \Delta([0,1])$ and let $0 \leq p < p' \leq 1$. We then have $\rev_Q(p') - \rev_Q(p) \leq \dem_Q(p)(p'-p) \leq p'-p$.
\end{lemma}

Throughout this work, we must handle distributional uncertainty over value distributions. To compare two distributions $P,Q \in \cP([0,1])$, we employ the \emph{Levy metric} defined by
\begin{equation}
\label{eq:levy-def}
    \dL(P,Q) \defeq \inf \{ \eps > 0 : \dem_P(x - \eps) - \eps \leq \dem_Q(x) \leq \dem_P(x + \eps) + \eps \: \forall x \in \R \}.
\end{equation}
This quantity is at most 1 and equals the side length of the largest square which can be inscribed between the graphs of $\dem_P$ and $\dem_Q$ (equivalently, the CDFs of $P$ and $Q$). For type distributions $D,D' \in \cP(\Theta)$, we use the Levy distance between their projected value distributions, taking $\dL(D,D') \defeq \sup_{u \in \unitsph} \dL(\proj(D,u),\proj(D',u))$. We use this metric because, if $D$ and $D'$ are close under $\dL$, then there exists a policy which performs well on both of them; this motivates the use of the L\'evy metric throughout the dynamic pricing literature (see, e.g., \citealp{leme2023pricing}).

\begin{lemma}[Pricing implication of L\'evy metric bound]
\label{lem:levy-metric-bd}
Suppose that $D,D' \in \Delta(\Theta)$ satisfy $\dL(D,D') < \eps$. Then the conservative best-response policy $\pi(u) = \max\{\br_D(u) - \eps,0\}$ satisfies $\rev_D(\pi(u),u) \geq \rev_D(\br_D(u),u) - \eps$ and $\rev_{D'}(\pi(u),u) \geq \rev_{D'}(\br_{D'}(u),u) - 3\eps$ for all $u \in \cU$.
\end{lemma}

%% file: sections/3-alg.tex
\section{Contextual Algorithm with Optimal Dependence on  \texorpdfstring{$d$}{d} and  \texorpdfstring{$T$}{T}}
\label{sec:alg}

We now develop statistically efficient (albeit computationally inefficient) algorithms for contextual pricing. In \cref{ssec:alg-finite}, we treat the simpler setting where $D_\star$ belongs to a finite model class $\cD$ and $K_\star$ is known.
In \cref{ssec:alg-infinite}, we remove these two assumptions, achieving regret $\widetilde{O}(K_\star\sqrt{dT})$. We also provide a regret lower bound of $\Omega(\sqrt{K_\star dT})$, even if $K_\star$ is known, thus proving the optimality of our regret bound's dependence on $d$ and $T$. Omitted proofs appear in \cref{app:alg-proofs}.

\subsection{Warm-Up: Heterogeneous Contextual Pricing with a Finite Model Class}
\label{ssec:alg-finite}

As a warm-up, we consider pricing when $D_\star$ belongs to a known, finite model class. 

\begin{assumption}
\label{assump:finite-model-class}
Assume that $D_\star \in \cD$, where $\cD \subseteq \Delta(\Theta)$ is finite and known to the seller.
\end{assumption}

\noindent This realizability assumption simplifies our analysis, and the resulting algorithm extends naturally to infinite classes. We employ optimistic posterior sampling (\OPS), originally studied for contextual bandits by \cite{zhang2022thompson} under the name ``Feel-Good Thompson Sampling.'' 
In our instantiation for contextual pricing, \OPS (\Cref{alg:OPS}) maintains a posterior distribution over models, initialized at prior $\mu_1 \in \Delta(\cD)$. At round $t$ with context $u_t \in \cU$, we sample a model $D_t \sim \mu_t$, play the best-response price $p_t = \br_{D_t}(u_t)$, and observe purchase feedback $y_t$. Then, for each candidate model $D \in \cD$, we update its posterior weight $\mu_{t+1}(D)$ according to the loss $\ell_\lambda(\proj(D,u_t),p_t,y_t)$, defined by
\begin{equation*}
    \ell_\lambda(Q,p,y) \defeq \underbrace{(y- \dem_Q(p))^2}_{\text{model mismatch}} - \underbrace{\lambda \rev_Q(\br_Q)}_{\text{optimism bias}}.
\end{equation*}
The ``model mismatch'' penalty captures the extent to which the observed demand $y_t$ differs from that predicted by $D$ when $p_t$ is played. In particular, as a function of $D \in \cD$, the expected model mismatch $\E_{y_t}[(y_t - \dem_D(p_t,u_t))^2 \mid p_t]$ is minimized only by those models $D$ which make the same prediction as $D_\star$, i.e., those for which $\dem_D(p_t,u_t) = \dem_{\star}(p_t,u_t)$. On the other hand, the ``optimism bias'' reduces the loss for models which have the potential to provide large revenue, ensuring that we perform sufficient exploration.\medskip

\begin{algorithm2e}[H]
\SetAlgoNoEnd
\caption{\OPS: Contextual Pricing with a Finite Model Class}
\label{alg:OPS}
\textbf{Input}: finite model class $\cD \subseteq \Delta(\Theta)$, support size $K \geq 1$\;
initialize uniform prior $\mu_1 = \Unif(\cD)$ and optimism strength $\lambda = \sqrt{\log(|\cD|)/K T}$\label{step:OPS-init}\;
\For{each round $t \in [T]$}{
    observe context $u_t$\;
    sample model $D_t \sim \mu_t$\;
    play $p_t = \br_{D_t}(u_t)$ and observe $y_t$\;
    update $\mu_{t+1}(D) \propto \mu_t(D) \exp\bigl(-\ell_\lambda(\proj(D,u_t),p_t,y_t)\bigr)$ for each $D \in \cD$\;
}
\end{algorithm2e}

\begin{theorem}
\label{thm:OPS-regret-bd}
Under \cref{assump:finite-model-class}, \OPS with $K = K_\star$ achieves regret $\widetilde{O}(\sqrt{K_\star T \log|\cD|})$.
\end{theorem}

The requirement of known $K_\star$ is imposed for simplicity and will be removed in \cref{ssec:alg-infinite}. To prove \cref{thm:OPS-regret-bd}, we employ a disagreement coefficient that controls the per-context complexity of balancing exploration and exploitation. In general, for an arbitrary measurable space $\cX$ and function class $\cF:\cX \to \R$, we define the \emph{disagreement coefficient} of $\cF$ by
\begin{equation}
    \dis(\cF) \defeq \sup_{\substack{\eps,\delta > 0}} \sup_{\nu \in \Delta(\cX)} \frac{\delta^2}{\eps^2} \,\PP_{p \sim \nu} \Bigl(\exists f\in\cF: \E_{q\sim\nu}[f(q)^2]\leq \eps^2 \land |f(p)| > \delta\Bigr).
\end{equation}
In our setting, each $f$ will measure the discrepancy between the demand function predicted by some model $D$ and that of the true model $D_\star$, for a fixed context (full details will appear shortly). While not the most primitive complexity measure, variants of this quantity have been successfully used to analyze a wide variety of structured bandits and active learning problems (see Remark~\ref{rem:existing-results}). The $\delta^2/\eps^2$ scaling was historically chosen so that $\dis(\cF)$ can be directly bounded by the domain size $|\cX|$. For our application, $\cX = [0,1]$ is the (infinite) price set and each function $f \in \cF$, induced by a model $D \in \cD$ and context $u \in \cU$, measures the discrepancy between the demand functions of $D$ and $D_\star$ after projection onto $u$, i.e., $f(p) = \dem_D(p,u) - \dem_\star(p,u)$. In particular, we set $$\dis(\cD, D_\star) \defeq \sup_{u \in \cU} \dis\bigl(\{\dem_D(\cdot,u) - \dem_\star(\cdot,u): D \in \cD\}\bigr).$$ By this definition, if $\dis(\cD, D_\star)$ is small and the seller plays price $p \sim \nu$ when faced with context $u$, it is unlikely for a model $D$ to disagree with $D_\star$ at $p$ if it is close to $D_\star$ under the $L^2(\nu)$ norm, i.e., if $\E_{q \sim \nu}[(\dem_D(q,u) - \dem_\star(q,u))^2]$ is small. In particular, playing $p \sim \nu$ guarantees that
\begin{equation*}
    \forall D \in \cD, \qquad \underbrace{|(\rev_D - \rev_\star)(p,u)| > \delta}_{\text{$D$ poorly models revenue at $p$}} \implies \underbrace{\E_{q \sim \nu}[|(\dem_D - \dem_\star)(q,u)|^2] > \eps^2}_{\text{$D$ incurs substantial least squares loss on average}}
\end{equation*}
with probability at least $1 - \frac{\eps^2}{\delta^2}\dis(\cD,D_\star)$. Since \OPS penalizes models with substantial least squares loss, while incentivizing exploration via its optimism bonus, we are able to show the following.

\begin{lemma}
\label{lem:OPS-regret-dis}
\sloppy Under \cref{assump:finite-model-class} with optimism strength $\lambda > 0$, \OPS (\cref{alg:OPS}) achieves regret $\widetilde{O}\bigl(\lambda\, \dis(\cD,D_\star) T + \log(|\cD|)/\lambda\bigr)$.
\end{lemma}
The proof in \cref{prf:OPS-regret-dis} combines the \OPS analysis of \cite{zhang2022thompson} with a decoupling lemma due to \cite{foster21statistical}. To control $\dis(\cD,D_\star)$, we show that each function of the form $\dem_D(\cdot,u) - \dem_\star(\cdot,u)$ can be decomposed into $K_\star+1$ non-increasing pieces. In \cref{prf:dis-monotonic}, we prove the following disagreement coefficient bound for non-increasing functions.

\begin{lemma}
\label{lem:dis-monotonic}
Let $\cF:[0,1]\rightarrow\R$ be the set of nonincreasing functions. Then
$\dis(\cF)\leq 2$.
\end{lemma}

Next, we examine the useful notion of \emph{composite} function classes. A function class $\cG:\cZ \to \R$ is called an $N$-composite of $\cF:\cX \to \R$ if there exists a disjoint partition $\cZ = \cZ_1\cup\dots\cup\cZ_N$ and mappings $\{h_i : \cZ_i \to \cX\}_{i \in [N]}$ such that each $g\in\cG$ can be decomposed as $g(x) = f_i(h_i(x))$ for all $x \in \cZ_i$ and $i \in [N]$, for some choice of $\{f_i : \cX \to \R\}_{i \in [N]}$. We show the following in \cref{prf:dis-composition}.

\begin{lemma}
\label{lem:dis-composition}
If $\cG$ is an $N$-composite of $\cF$, then $\dis(\cG) \leq N\dis(\cF)$.
\end{lemma}
With these results in hand, we bound $\dis(\cD,D_\star) = O(K_\star)$ and prove the theorem. Even though the action space is infinite, the disagreement coefficient matches that which would arise with $K_\star$ actions.

\begin{proof}[Proof of \cref{thm:OPS-regret-bd}]
\label{prf:OPS-regret-bd}
For each $u \in \cU$, the function  $\dem_\star(\cdot,u)$ is piecewise constant with $K_\star+1$ sections, since jumps can only occur at the projections of the $K_\star$ types. For any $D \in \cD$, the demand $\dem_D(\cdot, u)$ is monotonic, since increasing price always reduces demand. Thus, $\dem_D(\cdot,u)-\dem_\star(\cdot,u)$ is non-increasing on each of the $K_\star+1$ sections, and so the function classes defining $\dis(\cD,D_\star)$ are $(K_\star+1)$-composites of the non-increasing function class. Applying Lemmas~\ref{lem:dis-monotonic} and ~\ref{lem:dis-composition} then implies that $\dis(\cD,D_\star) \leq 2(K_\star+1)$. The theorem then follows by the regret bound of Lemma~\ref{lem:OPS-regret-dis} and our choice of $\lambda$.
\end{proof}

\begin{remark}[Comparison to Thompson sampling]
Standard Thompson sampling corresponds to the alternative choice of log losses: $\ell(Q,p,y) = \log \PP_{z \sim \Ber(\dem_Q(p))}[z = y] = y \log \dem_Q(p) + (1-y)\log (1 - \dem_Q(p)).$ In comparison, \OPS uses the squared loss (this is not essential but simplifies analysis) and an optimism bias towards models under which the seller can attain large revenue. This is crucial for obtaining frequentist (rather than Bayesian) regret bounds, as outlined in \cite{zhang2022thompson}. One appealing aspect of the log loss is that it eliminates models which predict that the observed feedback would never occur, mirroring the elimination-based methods for contextual pricing when $K_\star=1$. Thus, one natural question (beyond the current scope) is whether $\OPS$ with log loss achieves regret scaling logarithmically in $T$ when $K_\star = 1$.\smallskip
\end{remark}

\begin{remark}[Relation with existing results]
\label{rem:existing-results}
Variants of the disagreement coefficient and the related Alexander capacity are well-studied in the active learning and empirical process theory literature \citep{hanneke2014theory}. The version above was first considered by \cite{foster21instance}. \cite{foster21statistical} proved a regret bound which translates to $\widetilde{O}(\sqrt{\dis(\cD)T\log|\cD|})$ in our setting, matching \cref{thm:OPS-regret-bd}. However, the estimation-to-decisions (E2D) meta-algorithm which they employ is non-constructive, hence we apply OPS instead. In \cite{zhang2022thompson}, the regret of OPS is controlled by a distinct ``decoupling coefficient.'' Our proof of Lemma~\ref{lem:OPS-regret-dis} shows that a (slightly modified) decoupling coefficient is bounded by the disagreement coefficient.
\end{remark}

\subsection{The General Case}
\label{ssec:alg-infinite}

We now seek to eliminate the assumptions that $D_\star$ belongs to a finite class $\cD$ and that the support size $K_\star$ is known to the seller. For the first point, we loosen the requirement that $D_\star$ belongs to $\cD$ and take $\cD$ to be a large, finite cover of the full distribution space $\Delta(\Theta)$. Then, we replace the uniform prior $\mu_1$ with a non-uniform prior that places less weight on models with large support sizes. Ultimately, this will enable a choice of optimism strength $\lambda$ that is independent of $K_\star$, achieving our second goal. Unfortunately, if $D_\star$ is close but not equal to a model in $\cD$, our analysis of \OPS fails.\smallskip

To remedy this, we employ perturbed OPS (\POPS, \cref{alg:POPS}), an OPS variant with conservatively perturbed and discretized prices. This modified algorithm and its analysis require some new notation. Given a value distribution $Q \in \Delta([0,1])$, define the $\eps$-smoothed demand function $\dem^\eps_Q$ by
\begin{align*}
    \dem^\eps_Q(p) \defeq \E_{\delta \sim \Unif([0,\eps])}[\dem_Q(p - \delta)].
\end{align*}
Similarly, we let $\rev^\eps_Q(p) \defeq p \,\dem^\eps_Q(p)$. Define contextual extensions $\dem^\eps_D(p,u)$ and $\dem^\eps_D(p,u)$ as in the non-smoothed case. For discretization, write $\cP_\eps \defeq \eps \mathbb{N} \cap [0,1]$ for prices which are multiples of $\eps$ and let $\br_Q^\eps \defeq \argmax_{p \in \cP_\eps} \rev^\eps(p)$ (lifting to $\br_D^\eps(u)$ as in the non-smoothed case).\smallskip

Now, at each round $t \in [T]$, \POPS samples a model $D_t \sim \mu_t$ from the current posterior $\mu_t \in \Delta(\cD)$ and computes its (discretized) best response $\hat{p}_t = \br^\eps_{D_t}(u_t)$. Instead of posting price $\hat{p}_t$ directly, \POPS posts $p_t = \max\{\hat{p}_t - \delta_t, 0\}$, where $\delta_t \sim \Unif([0,\eps])$ is a small random perturbation. Due to this perturbation and discretization, we employ the modified loss $\ell^{\,\eps}_\lambda(\proj(D_t,u_t), \hat{p}_t, y_t)$, where
\begin{equation*}
    \ell^{\,\eps}_\lambda(Q, \hat{p}, y) \defeq \left(\dem^\eps_Q(\hat{p}) - y\right)^2 - \lambda \rev_Q^\eps(\br^\eps_Q).
\end{equation*}
The perturbations allow us, in the analysis of \POPS, to couple its trajectory when run with $D_\star \not\in \cD$ to a trajectory where $D_\star \in \cD$. The discretization is needed to bound a modified disagreement coefficient which appears in the analysis.
All together, we obtain the following.\medskip

\begin{algorithm2e}[H]
\SetAlgoNoEnd
\caption{Perturbed OPS (\POPS) for Contextual Pricing with Infinite Model Class}
    \label{alg:POPS}
    \textbf{Input}: discretization error $\eps \in [0,1)$, finite model cover $\cD \subseteq \Delta(\Theta)$, \hspace{4cm} \phantom{a}\hspace{9.75mm}model prior $\mu_1 \in \Delta(\cD)$, optimism strength $\lambda > 0$\;
    \For{each round $t \in [T]$}{
        observe context $u_t$\;
        sample model $D_t \sim \mu_t$ and perturbation strength $\delta_t \sim \Unif([0,\eps\,])$\;
        play $p_t = \max\{\hat{p}_t - \delta_t,0\}$, where $\hat{p}_t = \br^\eps_{D_t}(u_t)$ and observe $y_t$\;\label{step:perturbed-price}
        update $\mu_{t+1}(D) \propto \mu_t(D) \exp\bigl(-\ell^{\,\eps}_\lambda(\proj(D,u_t),\hat{p}_t,y_t)\bigr)$ for each $D \in \cD$\;
    }
\end{algorithm2e}\medskip

\begin{theorem}
\label{thm:pops}
With appropriately tuned parameters, \POPS (\cref{alg:POPS}) achieves regret $\widetilde{O}(K_\star\sqrt{dT})$ without prior knowledge of $K_\star$. Moreover, even for known $K_\star > 1$ and stochastic contexts, no contextual pricing policy can achieve expected regret $o(\sqrt{K_\star dT})$ for all instances if $T \geq 8d K_\star^3 \log(2d)$.
\end{theorem}

Our analysis views the perturbation at Step~\ref{step:perturbed-price} as being performed by nature, rather than the seller.
Treating the seller's action as $\hat{p}_t$, they then observe a purchase ($y_t=1$) with probability 
\begin{equation*}
    \E_{\delta_t}[\dem_\star(\max\{\hat{p}_t - \delta_t,0\},u_t) \mid \hat{p}_t, u_t] = \E_{\delta_t}[\dem_\star(\hat{p}_t - \delta_t,u_t) \mid \hat{p}_t, u_t] = \dem_\star^\eps(\hat{p}_t,u_t),
\end{equation*}
justifying the definitions above. Through this lens, \POPS can viewed as OPS for an alternative, smoothed demand model. To bound regret, we apply a three-step argument. 
\smallskip

First, we show that \POPS maintains our OPS regret bound when $D_\star \in \cD$. This requires bounding a modified decoupling coefficient and is the only step where discretization is used. A direct application of the previous OPS analysis provides a regret bound with respect to a smoothed and discretized benchmark. Fortunately, one-sided Lipschitzness of revenue (Lemma~\ref{lem:one-sided-lipschitzness}) ensures that this modified regret benchmark is within $O(\eps T)$ of the original benchmark, as we prove in \cref{prf:POPS-regret-finite}. 

\begin{lemma}
\label{lem:POPS-regret-finite}
Under \cref{assump:finite-model-class}, using prior $\mu_1 \in \Delta(\cD)$, discretization error $\eps \in [0,1)$, and optimism strength $\lambda > 0$, \POPS (\cref{alg:POPS}) achieves regret $\widetilde{O}\bigl(\lambda\, K_\star T + \log(1/\mu_1(D_\star))/\lambda + \eps T\bigr)$.
\end{lemma}

Next, we show that, if there exists $D \in \cD$ whose smoothed demand function uniformly approximates that of $D_\star$, then the trajectory of \POPS under $D_\star$ can be coupled with that under $D$, such that the trajectories coincide with high probability. See \cref{prf:POPS-coupling} for the proof.

\begin{lemma}
\label{lem:POPS-coupling}
If there exists $D \in \cD$ for which $\|\dem_D^\eps - \dem_{D_\star}^\eps\|_\infty \leq \eps$, then the trajectory $\{u_t,\hat{p}_t,y_t\}_{t=1}^T$ of \POPS with type distribution $D^\star$ can be coupled with that $\{u_t',\hat{p}_t',y_t'\}_{t=1}^T$ of \POPS with type distribution $D$, such that the two trajectories are identical with probability $1-\eps T$.
\end{lemma}

Finally, we show that, to obtain a uniform $\eps$-cover of the smoothed demand functions, it suffices to find a $O(\eps^2)$-cover of the type distributions under the L\'evy metric $\dL$, as defined in \eqref{eq:levy-def}. Moreover, we show that the family of all type distributions with support size at most $K$, $\Delta_K(\Theta)$, admits an appropriately small L\'evy cover. For notation, we write $N(X,\mathrm{d},\tau)$ for the size of the smallest subset $X' \subseteq X$ which covers set $X$ under metric $\mathrm{d}$ up to accuracy $\tau$ (i.e., for each $x \in X$, there exists $x' \in X'$ such that $\mathrm{d}(x,x') \leq \tau$). A proof of the following appears in \cref{prf:metric-entropy}.

\begin{lemma}
\label{lem:metric-entropy}
If $D,D' \!\in\! \Delta(\Theta)$ satisfy $\dL(D,D') \!\leq \!\eps^2/2$, then $\|\dem_D^\eps - \dem_{D'}^\eps\|_\infty \leq \eps$. Moreover, we have $\log N\bigl(\Delta_K(\Theta),\dL,\eps\bigr) = \widetilde{O}\left(dK\log 1/\eps\right)$.
\end{lemma}

In \cref{prf:pops-ub}, we combine these lemmas to prove the upper bound of \cref{thm:pops}. For the lower bound in \cref{prf:pops-lb}, we modify a construction from \cite{cesabianchi19pricing} for the non-contextual case, so that it can be tensored into $d$-dimensions without leaking information between orthogonal contexts.

\begin{remark}[Bayesian analysis]
Consider the Bayesian setting where $D_\star$ is sampled from a known prior $\mu \in \Delta(\cD)$ (keeping $\cD$ finite for simplicity). Then, Lemma~\ref{lem:POPS-regret-finite} with $\mu_1 = \mu$ and $\eps = 0$ implies that $\OPS$, with $\mu_1$ set to $\mu$ at Step~\ref{step:OPS-init}, achieves Bayesian regret $\widetilde{O}\bigl(\lambda K_\star T + \E_{D_\star \sim \mu}[-\log \mu(D_\star)]/\lambda\bigr) = \widetilde{O}(\lambda K_\star T + H(\mu)/\lambda)$, where $H$ is Shannon entropy. For known $K_\star$, $\lambda$ can be tuned to achieve regret $\widetilde{O}\bigl(\sqrt{K_\star T H(\mu)}\bigr)$, matching the (non-Bayesian) bound of Theorem~\ref{thm:OPS-regret-bd} with $H(\mu)$ instead of $\log|\cD|$.
\end{remark}

\begin{remark}[Misspecified/noisy setting]
We note that \POPS is inherently robust to small misspecifications and noise. Indeed, if $D_\star$ does not have support size $K_\star$ but is within L\'evy distance $\delta$ of the family $\Delta_{K_\star}(\Theta)$, then the proof of \cref{thm:pops} still goes through if we choose $\eps \gets T^{-2} + \sqrt{\delta}$, incurring an additive regret overhead of $\sqrt{\delta} T^2$. The same overhead applies (up to logarithmic factors) in the noisy model where valuations are subject to mean-zero $\delta^2$-sub-Gaussian noise, since the associated convolution can only perturb demands by $\widetilde{O}(\delta)$ under the L\'evy metric. We do not attempt to optimize this overhead but note that one is better off using EXP4 (as below) when $\delta \gg 1/\poly(T)$.
\end{remark}

\begin{remark}[Large $K_\star$]
Our lower bound above can be restated as $\widetilde{\Omega}(\min\{\sqrt{K_\star d T}, d^{1/3}T^{2/3}\})$. When $K_\star = \widetilde{\Omega}((T/d)^{1/3})$ and the second term is active, \POPS no longer has an advantage over EXP4 with na\"ively discretized actions. In particular, one can run EXP4 with a policy cover of log cardinality $\widetilde{O}(d \eps^{-2})$ and price set $\{\eps,2\eps, \dots, 1\}$, incurring regret overhead $\eps T$ due to discretization. This gives a regret bound of $\widetilde{O}(\sqrt{Td^2\eps^{-3}} + \eps T)$, which balances out to $d^{2/5}T^{4/5}$ after tuning. Characterizing the optimal regret in this large $K_\star$ regime is an interesting question beyond the current scope.
\end{remark}

%% file: sections/4-zooming.tex
\section{Non-Contextual Refinements via Zooming}
\label{sec:zooming}

Our results from \cref{sec:alg} leave a key open question: what is the optimal regret dependence on~$K_\star$? We %
resolve this question for the non-contextual setting, where $d=1$ and, without loss of generality, $u_t \equiv 1$ for all $t$. To do so, we employ adaptive discretization (aka \emph{zooming}) methods from Lipschitz bandits \citep{slivkins2008} with novel variance-aware confidence intervals and achieve a regret bound of $\widetilde{O}(\sqrt{K_{\star}T})$.
Throughout, we label the $K_\star$ types $\supp(D_\star) = \{\theta^{(1)} < \theta^{(2)} < \dots < \theta^{(K_\star)}\}$ and set $\theta^{(0)} = 0$.

Our algorithm, \Zooming, %
mirrors standard zooming \citep{slivkins2008} with two key adjustments. First, since revenue is only one-sided Lipschitz, we use a dyadic price selection rule inspired by \cite{podimata2021adaptive}. Second, our confidence intervals incorporate empirical variance, a method previously used for variance-aware $K$-armed bandits \citep{audibert2009exploration}.
In more detail, \Zooming maintains a set $S$ of active prices in $[0,1]$ and a variance-aware confidence interval for the expected revenue at each $p \in S$. Each active price ``covers'' 
an interval of neighboring, larger prices, with the width of this covering interval scaling proportionally to that of the confidence interval. The intuition is that a small increase in price can only marginally increase expected revenue, so it is not worth exploring such covered prices.
Initially, every price in $[0,1]$ is covered by some point in $S$. At each round, \Zooming optimistically chooses a price in $S$ and updates its confidence and covering intervals. If after the update there exists an uncovered price, then we add a new point to $S$ which covers it, maintaining the invariant that every price is covered.

\begin{theorem}
\label{thm:zooming}
\Zooming (\Cref{alg:zooming}) achieves regret $\widetilde{O}\bigl(\min\bigl\{\sqrt{K_\star T}, T^{2/3}\bigr\}\bigr)$ for non-contextual pricing%
, without knowledge of $K_\star$. This is minimax optimal up to logarithmic factors when $K_\star > 1$.
\end{theorem}

To bound regret, we employ a variance-aware zooming dimension which controls its performance. For comparison, we first recall the definition of the standard zooming dimension, which characterizes a certain complexity of the expected reward function $\rev_\star(p)$. For each $\delta > 0$, write $X_\delta \defeq \{ p \in [0,1] : \gap_\star(p) \leq \delta \}$ for the set of $\delta$-approximate revenue maximizers. Write $N(X,\delta) \defeq N(X,|\cdot|,\delta)$ for the smallest $\delta$-covering of a set $X \subseteq \R$. Then, for each $c > 0$, define the \emph{zooming dimension}
\begin{equation*}
    \ZoomDim(c) \defeq \inf\{z \geq 0 : N(X_\delta,\delta/10) \leq c \delta^{-z}\: \forall \delta > 0 \}.
\end{equation*}

Standard zooming techniques imply that \Zooming achieves regret $c^{1/(2+z)} T^{1-1/(2+z)}$ when $\ZoomDim(c) \allowbreak \leq z$, even with confidence intervals that do not incorporate empirical variance. 
Since the price interval $[0,1]$ is one-dimensional, one can bound $\ZoomDim(c) \leq 1$ for $c = O(1)$, giving regret $\smash{\widetilde{O}(T^{2/3})}$. Moreover, the set $X_\delta$ of approximate revenue maximizers is contained in the union of $K_\star$ intervals preceding the unknown types, where the interval corresponding to type $\theta^{(i)}$ has width $\delta/\dem_\star(\theta^{(i)}) \leq \delta/\dem_\star(\theta^{(K_\star)})$. This implies $\ZoomDim(c) = 0$ and gives regret $\widetilde{O}(\sqrt{cT})$, but only for $c = O(K_\star/\dem_\star(\theta^{(K_\star)}))$, which may be arbitrarily large for worst-case instances.

To remedy this, we incorporate variance, writing $\sigma^2(p) = p^2\dem_\star(p)(1-\dem_\star(p))$ for the revenue variance when $p$ is played. For the problematic types above with low demand, variance is also low, and the confidence intervals of \Zooming are designed to adapt to this. Specifically, our proof in \cref{prf:regret-zooming} shows that the regret of \Zooming scales according to a \emph{variance-aware zooming dimension}, defined as follows. First define the variance-weighted covering number $N_\mathrm{v}(X,\delta) \defeq \inf \bigl\{ \sum_{x \in X'} \sigma^2(x) : \text{$X'$ is a $\delta$-cover of } X \bigr\}$. Then, for each $c > 0$, we set
\begin{equation*}
    \ZoomDimV(c) \defeq \inf\{z \geq 0 : N_\mathrm{v}(X_\delta,\delta/10) \leq c \delta^{-z}\: \forall \delta > 0 \}.
\end{equation*}
Note that $\texttt{ZoomDimV}(c) \leq \texttt{ZoomDim}(c)$, since $\sigma^2(p) \leq 1$. Moreover, we show $\ZoomDim(10K_\star) = 0$, implying the desired $\widetilde{O}(\sqrt{K_\star T})$ regret. The lower bound follows from \cref{thm:pops} with $d=1$.

\begin{remark}[Comparison to \cite{cesabianchi19pricing}]
\label{rem:zooming-comparison}
The non-contextual setting was previously studied by
\cite{cesabianchi19pricing}, whose Algorithm 1 achieves regret $\widetilde{O}(\sqrt{K_\star T}) + V(V+1)$, where $V = \max_{i \in [K_\star]}\bigl(\theta^{(K_\star)}\bigr)^4\bigl(\theta^{(i)} - \theta^{(i-1)}\bigr)^{-5}$. They maintain a set of intervals which contain all types with substantial probability mass, gradually refining these intervals until they are all of width $O(T^{-1/2})$, at which point they employ UCB over the intervals' left endpoints. Unfortunately, the instance-dependent term $V(V+1)$ can blow up to infinity for worst-case realizations of $D_\star \in \Delta_{K_\star}([0,1])$, in contrast to our guarantee.
\end{remark}

%% file: sections/5-extras.tex
\section{Improved Performance with Ex-Post Type Observability}
\label{sec:extras}

We now study dynamic pricing with heterogeneous buyer types under the additional assumption that the learner can \textit{identify} each arriving type. That is, after setting price $p_t$, the learner observes the purchase feedback $\mathds{1}\{\langle u_t, \theta_t \rangle \geq p_t\}$ \emph{and} some information about the sampled type $\theta_t$. We consider the two models of observability. In the first, the learner only observes an identifier $z_t \in [K_\star]$ that specifies which of the $K_\star$ candidate types was drawn. In practice, the learner need not know $K_\star$ \emph{a priori}. Here, we design an algorithm that matches the $\widetilde O(K_\star \sqrt{dT})$ regret bound of \POPS and can be implemented efficiently, using a contextual search algorithm for $K_\star = 1$ as a subroutine.
In the second model, the learner observes the full type embedding $\theta_t \in \Theta$. Here, we show that best-responding to a simple plug-in estimate for $D_\star$ achieves an improved regret bound of~$\widetilde{O}(\sqrt{\min\{K_\star,d\}T})$.

\xhdr{Observed type identifiers.}
Our algorithm for the model where the learner only observes the identifier uses a $K_\star =1$ contextual search algorithm, \ProjectedVolume \citep{lobel18contextual}, as a subroutine. We maintain a separate instance of this \ProjectedVolume algorithm for each observed type and keep track of the empirical type frequencies along with the number of times we've explored each type. It then adaptively chooses which types to explore (or exploit) based on confidence estimates for the type distribution. We present the full algorithm and prove the following regret bound in Appendix~\ref{prf:unknloc_unknmas}.

\begin{theorem}\label{thm:unknloc_unknmas}
\sloppy Consider contextual dynamic pricing with ex-post type observability where the learner observes which type $z_t \in [K_\star]$ arrived. Then, Algorithm~\ref{alg:ident} achieves regret $\widetilde{O}(K_\star \sqrt{dT})$ and takes no more than time $\mathrm{poly}(K_\star, d, T)$ per round.
\end{theorem}

\xhdr{Observed type vectors.}
If the full type vector $\theta_t$ is revealed at the end of each round, we can achieve improved regret with a simpler algorithm. Indeed, writing $\hat{D}_t = \frac{1}{t}\sum_{\tau=1}^t \delta_{\theta_t}$ for the empirical type distribution after round $t$, we take each $p_t$ as the best response to $\hat{D}_{t-1}$ along the current context.

\begin{theorem}\label{thm:knloc_unknmas}
Consider contextual dynamic pricing with ex-post type observability where the learner observes  $\theta_t \in \Theta$ at the end of each round. Then the algorithm which plays $p_1 = 1/2$ and best response $p_t = \br_{\hat{D}_{t-1}}(u_t)$ for remaining rounds achieves regret $\smash{\widetilde{O}(\sqrt{\min\{K_\star, d\}T})}$. Each price can be computed in time $\poly(K_\star,d)$.
\end{theorem}

The proof in \cref{sec:knloc_unknmas} uses VC dimension bounds to show that the empirical revenue function $\smash{\rev_{\hat{D}_t}}$ converges uniformly in both arguments (price and context) to the true revenue function $\rev_\star$.

%% file: sections/6-discussion.tex
\section{Discussion}
\label{sec:discussion}

In this work, we have introduced contextual dynamic pricing with heterogeneous buyers. Our main algorithm achieves a regret bound of $\tilde{O}(K^\star \sqrt{dT})$, optimal up to a $O(\sqrt{K_\star})$ factor and logarithmic terms. Our analysis bounds the disagreement coefficient by leveraging a novel decomposition lemma for aggregate demand functions with $K_\star$ breakpoints, thereby ensuring an efficient exploration-exploitation tradeoff. Additionally, we propose a variance-aware zooming algorithm for the non-contextual pricing case, improving regret dependence on $K_\star$ by incorporating adaptive discretization methods from the Lipschitz bandits literature. Finally, under stronger observability assumptions on the buyers' types, we develop efficient algorithms that significantly reduce regret to $\tilde{O}(\sqrt{\min\{K_\star, d\}T})$, demonstrating the potential benefits of richer feedback in dynamic pricing settings.

There are several natural open questions, the first revolving around \emph{computation}. The run time of \POPS scales with the size of the discretized model class, which is exponential in $K_{\star}$ and $d$. It would be interesting to see if there is a way to alleviate this exponential dependence, while achieving similar regret bounds. The second question is around the optimal dependence on $K_{\star}$ for the general, contextual case. While in Section~\ref{sec:zooming} we showed how to optimize the dependence of our bounds on $K_{\star}$ it is unclear how to scale this approach for the contextual version of the problem. One starting point could be the results on Zooming techniques for contextual bandits (see e.g., \citealp{slivkins2011contextual}). Finally, it would be interesting to see if our results can be applied to more broad families of settings where a learner tries to learn from heterogeneous agents while obtaining only single-bit feedback: for example, it is unknown if the approach presented in this work generalizes to general contextual search settings (i.e., with $\eps$-ball or symmetric loss) or if it generalizes for settings that share some core properties with pricing, but differ in the fundamental techniques used to address them (see e.g., \citealp{ho2014adaptive}).\smallskip

%% file: sections/7-acknowledgements.tex
\xhdr{Acknowledgements.} 
The authors thank Jason Gaitonde and Oliver Richardson for helpful discussions on high-dimensional probability and algorithm design. We are also grateful to the Simons Institute for the Theory of Computing, as this work started during the Fall’22 semester-long program on Data-Driven Decision~Processes.

%% file: sections/A-prelims-proofs.tex
\section{Proofs for \texorpdfstring{\cref{sec:prelims}}{Section~\ref{sec:prelims}}}
\label{app:prelims-proofs}

\subsection{One-Sided Lipschitzness of Revenue (Proof of Lemma~\ref{lem:one-sided-lipschitzness})}
\label{prf:one-sided-lipschitzness}
We simply bound $\rev_Q(p') - \rev_Q(p) = p'\dem_Q(p') - p\dem_Q(p) \leq \dem_Q(p)(p'-p) \leq p'-p$, using monotonicity of demand functions.\qed

\subsection{Pricing Implication of L\'evy Metric Bound (Proof of Lemma~\ref{lem:levy-metric-bd})}
\label{prf:levy-metric-bd}
By the definition of $\dL$, it suffices to prove the lemma when $d=1$. The first revenue lower bound holds by Lemma~\ref{lem:one-sided-lipschitzness}. For the second, omitting the context $u$ and writing $[x]_+ = \max\{x,0\}$, we use the L\'evy metric guarantee to bound
\begin{align*}
    \rev_{D'}(\pi) &= [\br_D - \eps]_+\, \dem_{D'}([\br_D - \eps]_+)\\
    &\geq [\br_D - \eps]_+\, \dem_{D'}(\br_D - \eps) \tag{demand equals 1 for $p \leq 0$}\\
    &\geq [\br_D - \eps]_+\, (\dem_{D}(\br_D - 2\eps) - \eps) \tag{$\dL$ bound}\\
    &\geq [\br_D - \eps]_+\, (\dem_{D}(\br_D) - \eps) \tag{monotonicity of $\dem_D$}\\
    &\geq \rev_D(\br_D) - 2\eps\\
    &\geq \rev_D(\br_{D'}) - 2\eps \tag{$\br_D$ maximizes $\rev_D$}\\
    &= \br_{D'} \, \dem_D(\br_{D'}) - 2\eps\\
    &\geq \br_{D'} \, (\dem_{D'}(\br_{D'} - \eps) - \eps) - 2\eps \tag{$\dL$ bound}\\
    &\geq \br_{D'} \, \dem_{D'}(\br_{D'}) - 3\eps\\
    &= \rev_{D'}(\br_{D'}) - 3\eps,
\end{align*}
as desired. \qed

%% file: sections/B1-alg-proofs.tex
\section{Proofs for \texorpdfstring{\cref{sec:alg}}{Section 3}}
\label{app:alg-proofs}

To unify analysis of \cref{ssec:alg-finite,ssec:alg-infinite}, we introduce a more general problem setup and algorithm.

\subsection{Generalized Problem Setup}
\label{prf:generalized-setup}

To start, we replace the price set $[0,1]$ with a subset $\cP \subseteq [0,1]$, which will remain $[0,1]$ in \cref{ssec:alg-finite} but will be restricted to a finite set for \cref{ssec:alg-infinite}. Then, instead of selecting a type distribution $D_\star$, we have the adversary choose a demand function $f^\star$ which maps price $p \in \cP$ and context $u \in \cU$ to a purchase probability $f^\star(p,u) \in [0,1]$. Then, at round $t$, if the adversary selects context $u_t \in \cU$ and the learner posts price $p_t \in \cP$, purchase decision $y_t \in \{0,1\}$ is sampled independently from $\Ber(f^\star(p_t,u_t))$. This abstracts away our previous notions of buyer types and values and will also model the smoothed environment of \cref{ssec:alg-infinite}. We impose the corresponding notion of realizability.

\begin{setting}[realizability, general]
\label{setting:realizability-general}
Under the setup described above, the demand function $f^\star$ belongs to a known, finite class $\cF$ of measurable functions from $\cP \times \cU$ to $[0,1]$.
\end{setting}

Often, we shall fix a context and consider univariate (non-contextual) demand functions. Given a univariate demand function $g:[0,1] \to [0,1]$, we define the corresponding revenue function $\rev_g(p) = p \cdot g(p)$, best-response $\br_g = \argmax_{p \in \cP} \rev_g$ (breaking ties arbitrarily), and gap $\gap_g(p) = \rev_g(\br_g) - \rev_g(p)$. For a contextual demand function $f:[0,1] \times \cU \to [0,1]$ and a context $u \in \cU$, write $\proj(f,u)$ for the induced univariate demand function $p \mapsto f(p,u)$. We define $\rev_f(p,u) = \rev_{\proj(f,u)}(p)$, $\br_f(u) = \br_{\proj(f,u)}$, $\gap_f(p,u) = \gap_{\proj(f,u)}(p)$, and $\proj(\cF,u) \defeq \{\proj(f,u) : f \in \cF\}$, along with $\rev_\star \defeq \rev_{f^\star}$, $\br_\star \defeq \br_{f^\star}$, and $\gap_\star \defeq \gap_{f^\star}$. Finally, we define
\begin{equation*}
    \dis(\cF, f_\star) \defeq \sup_{u \in \cU} \dis\bigl(\{f(\cdot,u) - f_\star(\cdot,u): f \in \cF\}\bigr),
\end{equation*}
generalizing the definition in \cref{ssec:alg-finite}. Regret is defined as the sum of gaps $\sum_{t=1}^T\gap_\star(p_t,u_t)$.

\subsection{Generalized OPS and its Regret Guarantee}
\label{prf:generalized-OPS}

We now present a extension of \OPS and \POPS to the generalized setup of \cref{prf:generalized-setup}. Both can be recovered for appropriate choices of $\cF$ and $\cP$, which we will discuss later. First, for a univariate demand function $g:\cP \to [0,1]$, price $p$, and purchase decision $y$, we define loss 
\begin{equation}
\label{eq:generalized-loss}
    \ell_\lambda(g,p,y) \defeq (g(p) - y)^2 - \lambda \rev_g(\br_g).
\end{equation}
We now adapt OPS to this setting, introducing \GOPS (\cref{alg:generalized-OPS}).
\smallskip

\begin{algorithm2e}[H]
\caption{\GOPS: Generalized OPS for Contextual Pricing with Finite Model Class}
    \label{alg:generalized-OPS}
    \textbf{Input}: finite demand function class $\cF$,  model prior $\mu_1 \in \Delta(\cF)$, optimism strength $\lambda > 0$\;
    \For{each round $t \in [T]$}{
        observe context $u_t$\;
        sample demand function $f_t \sim \mu_t$\;
        play $p_t = \br_{f_t}(u_t)$ and observe $y_t$\;
        update $\mu_{t+1}(f) \propto \mu_t(f) \exp\bigl(-\ell_\lambda(\proj(f,u_t),p_t,y_t)\bigr)$ for each $f \in \cF$\;
    }
\end{algorithm2e}\smallskip

We prove the following regret bound.
\begin{lemma}
\label{lem:generalized-OPS-regret}
Under \cref{setting:realizability-general}, with model prior $\mu_1 \in \Delta(\cF)$ and optimism strength $\lambda \geq 4/T$, \GOPS (\cref{alg:generalized-OPS}) achieves regret $ 25\lambda (\dis(\cF,f_\star) \lor 1) T\log^2(T) + \log(1/\mu_1(f_\star))/\lambda$.
\end{lemma}

Our proof employs the following decoupling lemma.

\begin{lemma}[\citealp{foster21statistical}]
\label{lem: dc bound}
Let $\cG$ be a finite family of univariate demand functions and fix $g^\star \in \cG$. Then, for any $\nu \in \Delta(\cG)$ and $\gamma > 0$, we have
\begin{align*}
    \E_{g \sim \nu}\bigl[|g(\br_g) - g^\star(\br_g)|\bigr]\leq 6\frac{\dis(\cG)\log^2(\gamma\lor e)}{\gamma}+\gamma\E_{\tilde g, g\sim\nu}\bigl[(\tilde g(\br_g) - g^\star(\br_g))^2\bigr]\,.
\end{align*}
\end{lemma}

This is simply Lemma E.2 of \cite{foster21statistical} with function class $\{g \!-\! g^\star \!:\! g \in \cG\}$ and $\Delta \!\to\! 0$. In our proof, $\cG$ and $g^\star$ will be the projections of $\cF$ and $f^\star$ onto a fixed context $u \in \cU$.\smallskip

The remainder of our analysis is a slight modification to that of \cite{zhang2022thompson}, which we provide for completeness. For each round $t \in [T]$ of OPS, we adopt the following notation:
\begin{itemize}
    \item history up to round $t$: $S_t \defeq \{u_\tau,f_\tau,p_\tau,y_\tau\}_{\tau=1}^t$,
    \item true univariate demand function: $g^\star_t \defeq \proj(f^\star, u_t)$,
    \item univariate demand function posterior: $\nu_t \defeq \proj(\mu_t,u_t) \defeq \operatorname{Law}_{f \sim \mu_t}(\proj(f,u_t))$,
    \item sampled univariate demand function: $g_t \defeq \proj(f_t, u_t)$, so that $p_t = \br_{g_t}$,
    \item independently sampled univariate demand function (for analysis): $\tilde{g}_t \sim \nu_t$,
    \item regret: $\REG_t \defeq \rev_\star(\br_\star(u_t),u_t) - \rev_\star(p_t,u_t) = \rev_{g^\star_t}(p_t) - \rev_{g^\star_t}(p_t)$,
    \item least-squares errors: $\LS_t(g) \defeq (g(p_t) - g^\star(p_t))^2$,
    \item ``feel-good'' (optimism) bonuses: $\FG_t(g) \defeq \rev_{g}(\br_{g}) - \rev_{g^\star_t}(\br_{g^\star_t})$,
    \item loss discrepancies: $\Delta L_t(g) \defeq \ell_\lambda(g,p_t,y_t) - \ell_\lambda(g^\star_t,p_t,y_t),$
    \item potential function: $Z_t \defeq \E_{S_t} \log \E_{f \sim \Unif(\cF)}\exp\bigl(-\sum_{\tau=1}^t \Delta L_t(\proj(f,u_t))\bigr)$.
\end{itemize}

Our proof requires several supporting lemmas. The first is a basic concentration result.

\begin{lemma}
\label{lem:moment-generating-fn-bd}
For $c \geq 0$ and a random variable $X$ supported on $[0,1]$, we have $\log\E \exp(-cX) \leq (\frac{1}{2}c^2-c)\E X$. For $X$ supported on $[a,b]$, we have $\log \E \exp(cX) \leq c\E X + \frac{1}{8}(b-a)^2c^2$.
\end{lemma}
\begin{proof}
For the first inequality, we bound
\begin{equation*}
    \log\E \exp(-cX) \leq \E[\exp(-cX) - 1]\leq \E[-cX + \tfrac{1}{2}c^2X^2] \leq \E[-cX + \tfrac{1}{2}c^2 X] = (\tfrac{1}{2}c^2-c)\E X .
\end{equation*}
The second inequality is exactly Hoeffding's lemma.
\end{proof}

The next lemma mirrors Lemma 4 of \cite{zhang2022thompson}. This is a consequence of the definition of $\Delta L_t$ and the sub-Gaussianity of its components.
\begin{lemma}
For round $t$ of \GOPS (\cref{alg:generalized-OPS}), we have
\begin{align*}
    \log \E_{g \sim \nu_t}\E_{y_t|u_t,p_t}\exp\bigl(-\Delta L_t(g)\bigr) \leq -\frac{1}{4}\E_{g \sim \nu_t} \LS_t(g) + \lambda \E_{g \sim \nu_t} \FG_t(g) + \frac{3}{2}\lambda^2.
\end{align*}
\end{lemma}
We note that this lemma does not rely on how $p_t$ is selected.

\begin{proof}
Let $g \sim \nu_t$ and $y \sim \mathrm{Law}(y_t|u_t,p_t) = \Ber(g_t(p_t))$ be independent. Let $\eps = y - g^\star_t(p_t)$ denote the discrepancy between the observed and expected demand. Since demands lie in $[0,1]$, Lemma~\ref{lem:moment-generating-fn-bd} with $X = \eps$ gives
\begin{align}
\label{eq:sub-G}
    \E_{y} \exp\bigl(-2\eps\bigl(g^\star_t(p_t) - g(p_t)\bigl)\bigr) &\leq \exp\bigl(\tfrac{1}{2}\bigl(g^\star_t(p_t) - g(p_t)\bigl)^2\bigr) = \exp\bigl(\tfrac{1}{2} \LS_t(g)\bigr).
\end{align}
Moreover, we have
\begin{align*}
    -\Delta L_t(g) &= -\bigl(\eps + g^\star_t(p_t) - g(p_t)\bigr)^2 + \eps^2 + \lambda \FG_t(g)\\
    &= -2\eps\bigl(g^\star_t(p_t) - g(p_t)\bigr) - \LS_t(g) + \lambda \FG_t(g).
\end{align*}
Combining with \eqref{eq:sub-G} gives
\begin{equation*}
    \E_{y} \exp\bigl(-\Delta L_t(g)\bigr) \leq \exp\bigl(-\tfrac{1}{2} \LS_t(g) + \lambda \FG_t(g)\bigr).
\end{equation*}
Therefore, we have
\begin{align*}
    \log \E_{g,y} \exp\bigl(-\Delta L_t(g)\bigr) &\leq \log \E_{g} \exp\bigl(-\tfrac{1}{2} \LS_t(g) + \lambda \FG_t(g)\bigr)\\
    &\leq \frac{2}{3}\log \E_{g} \exp\bigl(-\tfrac{3}{4} \LS_t(g)\bigr) + \frac{1}{3}\log \E_{g} \exp\bigl(3\lambda\FG_t(g)\bigr),
\end{align*}
where the last inequality follows by Hölder's inequality. For the first term, we use Lemma~\ref{lem:moment-generating-fn-bd} with $X = \LS_t(g)$ and $c = \frac{3}{4}$ to bound
\begin{align*}
    \frac{2}{3}\log\left(\E_{g} \exp\bigl(-\tfrac{3}{4} \LS_t(g)\bigr)\right) \leq \frac{2}{3}\left(\frac{1}{2}c^2 -c\right)\E_g \LS_t(g) = -\frac{5}{16} \E_g \LS_t(g).
\end{align*}
For the second term, we apply the lemma with $X = \FG_t(g)$ and $c = 3\lambda$ to obtain
\begin{equation*}
    \frac{1}{3}\log \E_{g} \exp\bigl(3\lambda\FG_t(g)\bigr) \leq \lambda \E_{g} \FG_t(g) + \frac{3}{2}\lambda^2.
\end{equation*}
Combining, we have
\begin{equation*}
    \log \E_{g,y} \exp\bigl(-\Delta L_t(g)\bigr) \leq -\frac{5}{16} \E_{g} \LS_t(g) + \lambda \E_{g} \FG_t(g) + \frac{3}{2}\lambda^2,
\end{equation*}
implying the lemma.
\end{proof}

Our last helper lemma mirrors Lemma 5 of \cite{zhang2022thompson}.
\begin{lemma}
\label{lem:potential-diff-comparison}
For round $t$ of \GOPS (\cref{alg:generalized-OPS}), we have
\begin{equation*}
    \frac{1}{4\lambda} \E \LS_t(\tilde{g}_t) - \E \FG_t(g_t) \leq \frac{3}{2}\lambda + \frac{1}{\lambda}(Z_{t-1} - Z_t).
\end{equation*}
\end{lemma}
This lemma also does not rely on how prices are chosen.
\begin{proof}
Recall that $\mu_1 = \Unif(\cF)$. Defining $W_t(f \mid S_t) \defeq \exp\bigl(-\sum_{\tau=1}^t \Delta L_\tau(\proj(f,u_\tau))\bigr)$, we have $Z_t =  \E_{S_t}\log \E_{f \sim \mu_1} W_t(f \mid S_t)$. Note that
\begin{equation*}
    \mu_t(f) = \frac{W_{t-1}(f \mid S_{t-1})}{\E_{f \sim \mu_1} W_{t-1}(f \mid S_{t-1})} \mu_1(f).
\end{equation*}
We then bound
\begin{align*}
Z_t &= Z_{t-1} + \E_{S_t}\log \frac{\E_{f \sim \mu_1} W_t(f \mid S_t)}{\E_{f \sim \mu_1} W_{t-1}(f \mid S_{t-1})}\\
&= Z_{t-1} + \E_{S_t} \log \frac{\E_{f \sim \mu_1} W_{t-1}(f \mid S_{t-1}) \exp \bigl(-\Delta L_t(\proj(f, u_t))\bigr)}{\E_{f \sim \mu_1} W_{t-1}(f \mid S_{t-1})}\\
&= Z_{t-1} + \E_{S_t} \log \E_{f \sim \mu_t} \exp\bigl(-\Delta L_t\left(\proj(f, u_t)\right)\bigr)\\
&\stackrel{(a)}{\leq} Z_{t-1} + \E_{S_{t-1},u_t,p_t} \log \E_{y_t \mid u_t, p_t} \E_{g \sim \nu_t} \exp \bigl(-\Delta L_t(g)\bigr)\\
& \stackrel{(b)}{\leq} Z_{t-1}-\frac{1}{4}\E\LS_t(\tilde{g}_t) + \lambda \E \FG_t(g_t) + \frac{3}{2}\lambda^2,
\end{align*}
where (a) uses Jensen's inequality and (b) uses Lemma~\ref{lem:moment-generating-fn-bd}. Rearranging gives the lemma.
\end{proof}

Now, we return to Lemma~\ref{lem:generalized-OPS-regret}, where we will finally incorporate our price selection rule and the decoupling lemma.

\begin{proof}[Proof of Lemma~\ref{lem:generalized-OPS-regret}]
For each round $t \in [T]$, we recall that $p_t = \br_{g_t}$ and decompose
\begin{align*}
    \REG_t &= \rev_{g^\star_t}(\br_{g^\star_t}) - \rev_{g^\star_t}(p_t)\\
    &= \bigl[\rev_{g_t}(p_t) - \rev_{g^\star_t}(p_t)\bigr] - \bigl[\rev_{g_t}(\br_{g_t}) - \rev_{g^\star_t}(\br_{g^\star_t})\bigr]\\
    &= \bigl[\rev_{g_t}(p_t) - \rev_{g^\star_t}(p_t)\bigr] - \FG_t(g_t).
\end{align*}
Conditioning on $S_{t-1}$ and $u_t$, we apply Lemma~\ref{lem: dc bound} with $\gamma=\frac{1}{4\lambda}$ to obtain
\begin{align*}
    \E\bigl[\rev_{g_t}(p_t) - \rev_{g^\star_t}(p_t) \,\big|\, S_{t-1}, u_t \bigr] &= \E_{g \sim \nu_t}\bigl[\rev_{g}(\br_g) - \rev_{g^\star_t}(\br_g) \bigr]\\
    &\leq \E_{g \sim \nu_t}\bigl[|g(\br_g) - g^\star_t(\br_g)|\bigr]\\
    &\leq 24\lambda\,\dis(\cF,f_\star)\log^2(4\lambda^{-1}\lor e)+\frac{1}{4\lambda}\E_{\tilde g, g\sim\nu}\bigl[(\tilde g(\br_g) - g^\star_t(\br_g))^2\bigr]\\
    &= 24\lambda\,\dis(\cF,f_\star)\log^2(4\lambda^{-1}\lor e)+\frac{1}{4\lambda}\E\bigl[\LS_t(\tilde{g}_t) \,\big|\, S_{t-1},u_t \bigr].
\end{align*}
Taking expectations over $S_{t-1}$ and $u_t$, we bound
\begin{align*}
    \E\REG_t &\leq 24\lambda\,\dis(\cF,f_\star)\log^2(4\lambda^{-1}\lor e) + \frac{1}{4\lambda} \E\LS_t(\tilde{g}_t) - \E\FG_t(g_t) \\
    &\leq 24\lambda\,\dis(\cF,f_\star)\log^2(4\lambda^{-1}\lor e) + \frac{3}{2}\lambda + \frac{1}{\lambda}(Z_{t-1} - Z_t) ,
\end{align*}
using Lemma~\ref{lem:potential-diff-comparison}. Summing over $t \in [T]$ and noting that $Z_0 = 0$, we bound
\begin{align*}
    \E[R(T)] \leq T \left(24\,\lambda\dis(\cF,f_\star)\log^2(4\lambda^{-1}\lor e) + \frac{3}{2}\lambda \right) - \frac{1}{\lambda}Z_T.
\end{align*}
Moreover, by realizability, we have
\begin{align*}
    Z_T &\geq \E_{S_t}\log(\mu_1(f^\star)W_T(f^\star \mid S_t)) = \log( \mu_1(f_\star)).
\end{align*}
Combining, we obtain
\begin{align*}
    \E[R(T)] &\leq T \left(24\lambda\,\dis(\cF,f_\star)\log^2(4\lambda^{-1}\lor e) + \frac{3}{2}\lambda \right) - \frac{\log(\mu_1(f_\star))}{\lambda}\\
    &\leq T \left(25\lambda(\dis(\cF,f_\star) \lor 1)\log^2(4\lambda^{-1}\lor e) \right) - \frac{\log(\mu_1(f_\star))}{\lambda},\\
    &\leq T \left(25\lambda(\dis(\cF,f_\star) \lor 1)\log^2 T \right) - \frac{\log(\mu_1(f_\star))}{\lambda},
\end{align*}
as desired.
\end{proof}

\subsection{Base Regret Bound for \OPS (Proof of Lemma~\ref{lem:OPS-regret-dis})}
\label{prf:OPS-regret-dis}

Under the general setup of \cref{prf:generalized-setup}, we take $\cF$ to be the class of demand functions induced by $\cD$, set $\cP = [0,1]$, and fix $\mu_1 = \Unif(\cF)$. By these choices, \GOPS coincides exactly with \OPS, as does our notion of regret. Thus, Lemma~\ref{lem:generalized-OPS-regret} gives the desired regret bound of
\begin{equation*}
    T \left(25\lambda(\dis(\cD,D_\star) \lor 1)\log^2 T \right) - \frac{\log(\mu_1(D_\star))}{\lambda} = \widetilde{O}\bigl(\lambda T \dis(\cD,D_\star) + \log(|\cD|)/\lambda\bigr).\tag*{\qed}
\end{equation*}

\subsection{Disagreement Coefficient Bound for Non-increasing Functions (Proof of Lemma~\ref{lem:dis-monotonic})}
\label{prf:dis-monotonic}
Fixing $f \in \cF$, $\nu \in \Delta([0,1])$, and $p \in [0,1]$, suppose that $\E_{q\sim\nu}[f(q)^2]\leq \eps^2$ and $|f(p)| > \delta$. If $f(p) > \delta$, then $f(q) > \delta$ for all $q \leq p$ by monotonicity. Thus, $\PP_{q\sim\nu}(q\leq p)\delta^2 \leq \E_{q\sim\nu}[f(q)^2]\leq \eps^2$. Otherwise, if $f(p)<-\delta$, we analogously have $\PP_{q\sim\nu}(q\geq p)\delta^2 \leq \E_{q\sim\nu}[f(q)^2]\leq \eps^2$. Thus, for $\nu\in\Delta(\cX)$, we have
\begin{align*}
    &\PP_{p \sim \nu} \bigl(\exists f\in\cF: \E_{q\sim\nu}[f(q)^2]\leq \eps^2 \land |f(p)| > \delta\bigr)\\
    \leq\, &\PP_{p \sim \nu} \left(\PP_{q\sim\nu}(q\leq p)\leq \frac{\eps^2}{\delta^2} \lor \PP_{q\sim\nu}(q\geq p)\leq \frac{\eps^2}{\delta^2}\right)\\
    \leq\, &\PP_{p \sim \nu} \left(\PP_{q\sim\nu}(q\leq p)\leq \frac{\eps^2}{\delta^2}\right)+\PP_{p\sim\nu}\left( \PP_{q\sim\nu}(q\geq p)\leq \frac{\eps^2}{\delta^2}\right)\\
    \leq\, &2\frac{\eps^2}{\delta^2}\,.
\end{align*}
Plugging this into the definition of $\dis$ finishes the proof.\qed

\subsection{Disagreement Coefficient Bound for Composite Classes (Proof of Lemma~\ref{lem:dis-composition})}
\label{prf:dis-composition}
For any distribution $\nu \in \Delta(\cZ)$, write $\nu_i = h_i\circ\nu|_{\cZ_i}$ for law of $h_i(p)$ when $p \sim \nu$, conditioned on $p \in \cZ_i$, and let $\mu_\nu(i) = \PP_{p\sim\nu}(p\in\cZ_i)$. We then bound
\begin{align*}
    \dis(\cG) &= \sup_{\substack{\eps,\delta > 0}} \sup_{\nu \in \cP(\cZ)}\E_{i\sim \mu_\nu} \frac{\delta^2}{\eps^2} \,\PP_{p \sim \nu|_{\cZ_i}} \Bigl(\exists g\in\cG: \E_{q\sim\nu}[g(q)^2]\leq \eps^2 \land |g(p)| > \delta\Bigr)\\
    &\leq\sup_{\substack{\eps,\delta > 0}} \sup_{\nu \in \cP(\cZ)}\E_{i\sim \mu_\nu} \frac{\delta^2}{\eps^2} \,\PP_{p \sim \nu_i} \Bigl(\exists f\in\cF: \mu_\nu(i)\E_{q\sim\nu_i}[f(q)^2]\leq \eps^2 \land |f(p)| > \delta\Bigr)\\
    &\leq\sup_{\substack{\eps,\delta > 0}} \sup_{\nu \in \cP(\cZ)}\frac{\delta^2}{\eps^2}\E_{i\sim \mu_\nu}\left[\frac{\eps^2}{\delta^2\mu_{\nu}(i)}\dis(\cF) \right]\\
    & = N\dis(\cF),
\end{align*}
as desired. Here, the first inequality uses that $\E_{q \sim \nu}[g(q)^2] \geq \mu_\nu(i) \E_{q \sim \nu_i}[f(h_i(x)]$ for some $f \in \cF$, and the second uses that $\nu_i \in \Delta(\cX)$ and the definition of $\dis$.\qed

\subsection{Base Regret Bound for \POPS (Proof of Lemma~\ref{lem:POPS-regret-finite})}
\label{prf:POPS-regret-finite}

\xhdr{POPS as generalized OPS.} We observe that \POPS is an instance of \GOPS (\cref{alg:generalized-OPS}), with discretized price set $\cP = \cP_\eps$ and smoothed demand function class $\cF = \{ \dem_D^\eps : D \in \cD\}$ (where each $f \in \cF$ is viewed as a function on $\cP_\eps \times \cU$ rather than $[0,1] \times \cU$). Here, we view each $\hat{p}_t$ as the posted price instead of $p_t$. Indeed, taking $f^\star = \dem_\star^\eps$, we have
\begin{equation*}
    \E[y_t|\hat{p}_t, u_t] = f^\star(\hat{p}_t, u_t)
\end{equation*}
and, for value distribution $Q \in \Delta([0,1])$ with smoothed demand function $g = \dem_Q^\eps$, we have
\begin{equation*}
    \ell_\lambda^\eps(Q, \hat{p}_t, u_t) = \ell_\lambda(g, \hat{p}_t, u_t),
\end{equation*}
where $\ell_\lambda$ on the right hand size is defined in \eqref{eq:generalized-loss}. We do note, however, that the regret benchmark with smoothed demands and discretized prices differs slightly from the original benchmark. 

\xhdr{Fixing the regret benchmark.} Applying Lemma~\ref{lem:generalized-OPS-regret} for this choice of $\cP$ and $\cF$, we have that
\begin{align*}
    \E\left[\sum_{t=1}^T \rev_\star^\eps(\br_\star^\eps(u_t),u_t) - \rev_\star^\eps(\hat{p}_t,u_t)\right] \leq T \left(25\lambda(\dis(\cF,f_\star) \lor 1)\log^2 T \right) - \frac{\log(\mu_1(D_\star))}{\lambda}.
\end{align*}
Note that
\begin{align*}
     \E\left[\sum_{t=1}^T \rev_\star^\eps(\hat{p}_t,u_t)\right] = \E\left[\sum_{t=1}^T \rev_\star(p_t,u_t)\right],
\end{align*}
so the regret bound above measures cumulative revenue in line with our original regret definition. Although the benchmark does not match that in the original definition, we have $|\rev_\star^\eps(\br_\star^\eps(u_t),u_t) - \rev_\star(\br_\star(u_t),u_t)| = O(\eps)$ for all $t \in [T]$ due to one-sided Lipschitzness (Lemma~\ref{lem:one-sided-lipschitzness}). Consequently, the regret of POPS is bounded by 
\begin{align*}
    \E[R(T)] = \widetilde{O}\left( T \lambda(\dis(\cF,f_\star) \lor 1) - \frac{\log(\mu_1(D_\star))}{\lambda} + \eps T\right).
\end{align*}

It remains to bound the disagreement coefficient by $O(K_\star)$, giving the lemma. Our argument below mirrors that in the proof of \cref{thm:OPS-regret-bd} but takes into account the smoothing and discretization.

\xhdr{Bounded disagreement coefficient.} 
For each $u \in \cU$, the function  $\dem_\star^\eps(\cdot,u)$ with domain $\cP_\eps$ is piecewise constant with $O(K_\star)$ sections. Indeed, the unsmoothed demand function $\dem_\star(\cdot,u)$ is piecewise constant with $O(K_\star)$ sections, and smoothing can only introduce new sections at the $O(K_\star)$ prices in $\cP_\eps$ that are within distance $\eps$ of a previous section boundary. Moreover, smoothing preserves monotonicity of demand functions. Hence, the function classes defining $\dis(\cF,f_\star)$ are $O(K_\star)$-composites of the nonincreasing function class. Applying Lemmas~\ref{lem:dis-monotonic} and \ref{lem:dis-composition} then implies that $\dis(\cF,f_\star) = O(K_\star)$, giving the lemma. \qed

\subsection{Trajectory Coupling (Proof of Lemma~\ref{lem:POPS-coupling})}
\label{prf:POPS-coupling}

Fix any round $t$ with context $u_t$ and best-response price $\hat{p}_t$. Since $\|f_D^\eps - f_{D_\star}^\eps\|_\infty \leq \eps$, feedback $y_t$ coincides with that which would have been obtained if $D_\star = D$ with probability at least $1-\eps$, conditioned on $u_t$ and $\hat{p}_t$. Since the update to $\mu_t$ is only a function of $u_t$, $\hat{p}_t$, and $y_t$ (notably, not the realized price $p_t$), we can iterate through all rounds and apply a union bound to obtain the lemma. \qed

\subsection{Metric Entropy Bound (Proof of Lemma~\ref{lem:metric-entropy})}
\label{prf:metric-entropy}

For part one, fix $D,D' \in \Delta(\Theta)$ with $\dL(D,D') \leq \eps^2/2$. Then, for all $u \in \cU$ and $\hat{p} \in [0,1]$, we have
\begin{align*}
    f_D^\eps(\hat{p},u) = \E_{\delta \sim \Unif([0,\eps])}\left[\dem_D(\max\{\hat{p} - \delta,0\},u)\right]
\end{align*}
Note that the maximum is unneeded since $\proj(D,u)$ is supported on $[0,1]$ and places no mass on negative values. Writing $Q = \proj(D,u)$ and $Q' = \proj(D',u)$, we then have
\begin{align*}
    |f_D^\eps(\hat{p},u) - f_{D'}^\eps(\hat{p},u)| &= \left|\E_{\delta \sim \Unif([0,\eps])}\left[\dem_Q(\hat{p} - \delta)-\dem_{Q'}(\hat{p} - \delta)\right]\right|\\
    &\leq \frac{1}{\eps} \,\left| \int_0^1 \dem_Q(\hat{p} - \delta)-\dem_{Q'}(\hat{p} - \delta) \, \dd \delta \right|\\
    &\leq \frac{1}{\eps} \,\left|\int_0^1 \dem_Q(t)-\dem_{Q'}(t) \,\dd t\right|\\
    &= \frac{1}{\eps} \,\left|\int_0^1 \dem_Q(t)-\dem_{Q'}(t) \,\dd t\right|.
\end{align*}
Writing $\tau = \eps^2/2$, we further have
\begin{align*}
    \int_0^1 \dem_Q(t)-\dem_{Q'}(t) \,\dd t &= \int_{\tau}^{1+\tau} \dem_Q(t-\tau)-\int_0^1\dem_{Q'}(t) \,\dd t\\
    &\leq \int_{\tau}^{1} \dem_Q(t-\tau)-\dem_{Q'}(t) \,\dd t + \tau\\
    &\leq 2\tau,
\end{align*}
where the last step uses the fact that $\dL(Q,Q') \leq \tau$. A symmetric argument gives the reverse bound. Consequently, we have $|f_D^\eps(\hat{p},u) - f_{D'}^\eps(\hat{p},u)| \leq \frac{1}{\eps} \cdot 2\tau = \eps$, as desired.\medskip

For part two, let $\cC = \{ C_1, C_2, \dots , C_n\}$ denote the intersection of the standard partition of $\R^d$ into cubes of side length $\eps/\sqrt{d}$ with $\Theta \subseteq [0,1]^d$. Denote the lexicographically smallest vertex of each $C_i$ by $c_i$, and note that $\log(n) = O(d\log(d/\eps))$. Given any $D \in \Delta_K(\Theta)$, we define the initial discretization $\hat{D}_0 = \sum_{i=1}^n D(C_i) \delta_{c_i}$. We obtain the final discretized measure $\hat{D}$ by rounding each weight to a neighboring multiple of $\eps/K$ (choice doesn't matter so long as we maintain unit mass, this is always possible). Then, for any context $u \in \unitsph$ and price $p \in [0,1]$, we have
\begin{align*}
\PP_{\theta \sim D}(\langle u,\theta \rangle \leq p) - \PP_{\theta \sim \hat{D}}(\langle u,\theta\rangle \leq p + \eps) &\leq \eps\\
\PP_{\theta \sim \hat{D}}(\langle u,\theta\rangle \leq p - \eps) - \PP_{\theta \sim D} (\langle u,\theta\rangle \leq p) &\leq \eps,
\end{align*}
using that each cube in $\cC$ has diameter $\eps$ and that the mass in each cube was perturbed by at most $\eps/K$ (so with $K$ cubes the probability of any event is shifted by at most $\eps$). Thus, $\dL(D,\hat{D}) \leq \eps$. Using balls and bins, we can thus bound
\begin{equation*}
N(\Delta_K(\Theta),\dL,\eps) \leq n^K \binom{K + K/\eps}{K} \leq n^{K} (2K/\eps)^K = \exp(O(Kd\log(d/\eps) + K\log(K/\eps)),
\end{equation*}
as desired.\qed

\subsection{Proof of \texorpdfstring{\cref{thm:pops}}{Theorem~\ref{thm:pops}}, Upper Bound}
\label{prf:pops-ub}

\xhdr{Upper bound.} Fix $\eps = T^{-2}$ and $\lambda = \sqrt{d/T}$. We now construct $\cD$ and $\mu_1$. Write $M = \lceil \log T \rceil$, and, for $i = 1, \dots, M$, take $\cD_i$ to be a minimal $(\eps^2/2)$-cover of $\Delta_{2^i}(\Theta)$ under the L\'evy metric. By Lemma~\ref{lem:metric-entropy}, we have $\log |\cD_i| = \widetilde{O}(2^i d\log 1/\eps)$. Now set $\cD = \cD_1 \cup \dots \cup \cD_M$ and take $\mu_1(D) \propto (2^i|\cD_i|)^{-1}$ for $D \in \cD_i$. This ensures that $\log(1/\mu_1(D)) = \widetilde{O}(2^i d \log T)$ for $D \in \cD_i$.

Assume without loss of generality that $M \geq \log K_\star$; otherwise, the regret bound is vacuous. Then, there exists $\hat{D} \in \cD_{\lceil \log K_\star \rceil} \subseteq \cD$ such that $\dL(\hat{D},D_\star) \leq \eps^2/2$ and $\|\dem_D^\eps - \dem_{D_\star}^\eps\|_\infty \leq \eps$, again using Lemma~\ref{lem:metric-entropy}. Thus, by Lemma~\ref{lem:POPS-coupling}, the realized trajectory of \POPS $\{u_t,\hat{p}_t,p_t,y_t\}_{t=1}^T$ can be coupled with an alternative trajectory $\{u_t',\hat{p}_t',p'_t,y_t'\}_{t=1}^T$ of \POPS with type distribution $\hat{D}$, such that $\{u_t,\hat{p}_t,y_t\}_{t=1}^T = \{u_t',\hat{p}_t',y_t'\}_{t=1}^T$ with probability at least $1-\eps T$.
By Lemma~\ref{lem:POPS-regret-finite}, we have
\begin{equation*}
    \E\left[\sum_{t \in [T]} \rev_{\hat{D}}(\br_{\hat{D}}(u_t'),u_t') - \rev_{\hat{D}}(p_t',u_t')\right] = \widetilde{O}\bigl(K\sqrt{dT} - \sqrt{T/d} \cdot \log \mu_1(\hat{D})\bigr) = \widetilde{O}(K_\star\sqrt{dT}).
\end{equation*}
At this point, we can use the coupling guarantee and the bound $\dL(\hat{D},D_\star) \leq \eps^2$ to show that the left hand side above and the true expected regret differ by $O(\eps^2 T) = O(1)$, giving the theorem.

Specifically, by the coupling guarantee, we have
\begin{align*}
    \E[R(T)] &\leq  \E\left[\sum_{t=1}^T \rev_\star(\br_\star(u_t),u_t) - \rev_\star(p_t,u_t) \,\Bigg|\, \cE\right] + \eps T^2\\
    &= \E\left[\sum_{t=1}^T \rev_\star(\br_\star(u'_t),u'_t) - \rev_\star(p_t,u'_t) \,\Bigg|\, \cE \right] + 1
\end{align*}
Since $\dL(\hat{D},D_\star) \leq \eps^2$, Lemma~\ref{lem:levy-metric-bd} implies that $\rev_\star(\br_\star(u'_t),u'_t) \leq \rev_{\hat{D}}(\br_{\hat{D}}(u_t'),u_t') + O(\eps^2)$ for all rounds $t$. We further have
\begin{align*}
    \E[\rev_\star(p_t,u'_t) \mid u'_t, \hat{p}_t] &= \rev_{D_\star}^\eps(\hat{p}_t,u'_t)\\
    &\leq \rev_{\hat{D}}^\eps(\hat{p}_t,u'_t) + \eps\\
    &= \E[\rev_{\hat{D}}(p'_t,u'_t) \mid u'_t,\hat{p}'_t = \hat{p}_t] + \eps\\
    &\leq \E[\rev_{\hat{D}}(p'_t,u'_t) \mid u'_t,\hat{p}'_t,\cE] + \eps + \eps T.
\end{align*}
Combining the above, we obtain
\begin{align*}
    \E[R(T)] &\leq \E\left[\sum_{t=1}^T \rev_\star(\br_\star(u'_t),u'_t) - \rev_\star(p_t,u'_t) \,\Bigg|\, \cE \right] + \eps T^2\\
    &\leq \E\left[\sum_{t=1}^T \rev_{\hat{D}}(\br_{\hat{D}}(u'_t),u'_t) - \rev_\star(p_t,u'_t) \,\Bigg|\, \cE \right] + O(\eps T^2)\\
    &= \E\left[\sum_{t=1}^T \rev_{\hat{D}}(\br_{\hat{D}}(u'_t),u'_t)\,\Bigg|\, \cE \right] - \E\left[\sum_{t=1}^T\E\left[\rev_\star(p_t,u'_t) \mid u_t', \hat{p}_t\right] \,\Bigg|\, \cE\right]  + O(1)\\
    &\leq \E\left[\sum_{t=1}^T \rev_{\hat{D}}(\br_{\hat{D}}(u'_t),u'_t)\,\Bigg|\, \cE \right] - \E\left[\sum_{t=1}^T\E\left[\rev_{\hat{D}}(p_t',u'_t) \mid u_t', \hat{p}_t', \cE\right]\,\Bigg|\, \cE\right]  + O(1)\\
    &\leq \E\left[\sum_{t=1}^T \rev_{\hat{D}}(\br_{\hat{D}}(u'_t),u'_t) - \sum_{t=1}^T \rev_{\hat{D}}(p_t',u'_t) \,\Bigg|\, \cE\right]  + O(1)\\
    &\leq \E\left[\sum_{t=1}^T \rev_{\hat{D}}(\br_{\hat{D}}(u'_t),u'_t) - \sum_{t=1}^T \rev_{\hat{D}}(p_t',u'_t) \right]  + O(1)\\
    &= \widetilde{O}(K_\star\sqrt{dT}),
\end{align*}
as desired.\qed

%% file: sections/B2-lb-proofs.tex
\subsection{Proof of \texorpdfstring{\cref{thm:pops}}{Theorem~\ref{thm:pops}}, Lower Bound}
\label{prf:pops-lb}

Previously, \cite{cesabianchi19pricing} gave a lower bound of $\Omega(\sqrt{K_\star T})$ for the non-contextual case. We now modify their one-dimensional construction so that it can be cleanly tensored into $d$ dimensions, when $K_\star \geq 4$. \medskip

\xhdr{One-dimensional construction ($K_\star \geq 4$).} Starting in 1D, we define
valuations $\frac{1}{2} = v_1 \leq \dots \leq v_{K_\star} = 1$ by
$v_i \defeq \frac{1}{2} + \frac{i-1}{4K_\star - 2i - 2}$.
Define the base distribution $Q_0$ on $\{v_1, \ldots, v_{K_\star} \}$ by
\begin{align*}
    Q_0(v_i) \defeq \begin{cases}
        \frac{1}{2K_\star - 2}, & i \in [K_\star - 1]\\
        \frac{1}{2}, & i = K_\star
    \end{cases}.
\end{align*}
Observe that each valuation $v_i$ has the same expected revenue of $1/2$. Indeed, we compute
\begin{align}
    \dem_{Q_0}(v_i) &= \sum_{j \geq i} Q_0(v_j)\nonumber\\
    &= \left(K_\star - i\right) \cdot \frac{1}{2K_\star - 2} + \frac{1}{2}\nonumber\\
    &= \frac{2K_\star - i - 1}{2K_\star - 2}\label{eq:lb-demand}
\end{align}
and 
\begin{align*}
    \rev_{Q_0}(v_i) &= v_i \dem_{Q_0}(v_i) = \left(\frac{2K_\star - 2}{4K_\star - 2i - 2}\right) \cdot \left(\frac{2K_\star - i - 1}{2K_\star - 2}\right) = \frac{1}{2}.
\end{align*}

Now, for $j \in \{2, \ldots, K_\star - 1\}$, we define distribution $Q_j$ by slightly lowering the probability of $v_{j-1}$ and increasing the probability of $v_j$ by the some small $\eps > 0$, to be determined. That is, we define
\begin{align*}
    Q_j(v_i) \defeq \begin{cases}
        \frac{1}{2K_\star - 2} - \eps, & i = j-1\\
        \frac{1}{2K_\star - 2} + \eps, & i = j\\
        \frac{1}{2}, & i = K_\star\\
        \frac{1}{2K_\star - 2}, &\text{o.w.}
    \end{cases},
\end{align*}
which is well-defined so long as $\eps \leq \frac{1}{2K_\star - 2}$. In contrast to the construction in \cite{cesabianchi19pricing}, our $Q_j$ distributions share the same multiset of probability weights, differing only in the locations of the perturbed valuations. This allows us to tensor these problem instances into a $d$-dimensional instance without leaking information between instances.

Moreover, the hardness result of \cite{cesabianchi19pricing} is maintained, which applies even if $K_\star$ is known to the learner. The proof is quite similar, so we defer it to \cref{prf:1d-lower}.

\begin{lemma}[Regret of one-dimensional family, $K_\star \geq 4$]
\label{lem:1d-lower}
Fix any number of types $K_\star \geq 4$ and $T \geq K_\star^3$. Let $\cA$ be any algorithm for non-contextual pricing. Then, for tuned $\eps$ and an instance $Q_\star$ drawn uniformly at random from $\{Q_j\}_{j \in \{2, \dots, K_\star - 1\}}$, $\cA$ suffers expected regret at least $\Omega(\sqrt{K_\star T})$.
\end{lemma}

\xhdr{Tensoring one-dimensional instances.}
To extend these one-dimensional distributions into $d$ dimensions, define the base distribution $D_0 \in \Delta([0,1]^d)$ as follows. For $i \in [K_\star]$, let $\theta_i \defeq [v_i, \ldots, v_i] \in \mathbb{R}^d$, $w_i \defeq Q_0(v_i),$ and take $D_0 \defeq \sum_{i=1}^{K_\star} w_i \delta_{\theta_i}$.
That is, $\theta_i \in [0,1]^d$ has all entries equal to $v_i$ and $w_i$ is the probability $D_0$ places on $\theta_i$, taken to match $Q_0$ at $v_i$. By design, the marginal distribution of $D_0$ along each coordinate is $Q_0$. Now, for a selection $j = (j_1, \ldots, j_d) \in \{2, \dots, K_\star-1\}^d$, define the perturbed instance $D_j$ by starting from $D_0$ and modifying it as follows:
\squishlist
\item Adjust the probabilities $w_1$ and $w_2$ by $w_1 \gets w_1 - \eps$ and $w_2 \gets w_2 + \eps$.
\item For each dimension $\ell \in [d]$, permute the $\ell$th coordinates of the $\theta_i$ vectors so that the marginal distribution of their $\ell$th coordinates coincides with $Q_{j_\ell}$.
\squishend
Specifically, we define $D_j \defeq \sum_{i=1}^{K_\star} \tilde{w}_i \delta_{\tilde{\theta}_{j,i}}$ where
\begin{align*}
    \tilde{w}_i \defeq \begin{cases}
        \frac{1}{2K_\star - 2} - \eps, & i = 1\\
        \frac{1}{2K_\star - 2} + \eps, & i = 2\\
        \frac{1}{2}, & i = K_\star\\
        \frac{1}{2K_\star - 2}, &\text{o.w.}
    \end{cases}
\end{align*}
and $\tilde{\theta}_{j,i}[\ell] \defeq v_{\sigma(i)}$, for any permutation $\sigma$ of $[K_\star]$ such that $\sigma(j_\ell - 1) = 1$, $\sigma(j_\ell) = 2$, and $\sigma(K_\star) = K_\star$. Simply swapping $j_\ell - 1$ and $1$ and $j_\ell$ with $2$ works unless $j_\ell = 3$, in which case one can send $2 \to 1$, $3 \to 2$, and $1 \to 3$. As claimed above, this construction ensures that for each dimension $\ell \in [d]$, the marginal distribution of $D_j$ is $Q_{j_\ell}$. Indeed, for each $i \in [K_\star]$, we have
\begin{align*}
    \PP_{\theta \sim D_j}\left(\theta[\ell] = v_i\right) &= \begin{cases}
        \frac{1}{2K_\star - 2} - \eps, & i = j_\ell - 1\\
        \frac{1}{2K_\star - 2} + \eps, & i = j_\ell\\
        \frac{1}{2}, & i = K_\star\\
        \frac{1}{2K_\star - 2}, &\text{o.w.}
    \end{cases}\\
    &= Q_{j_\ell}(v_i).
\end{align*}

\xhdr{Lower bounding the regret.}
For the contextual setting, we sample a selection \( j = (j_1, \ldots, j_d) \in \{2, \dots, K_\star - 1\}^d \) uniformly at random and set $D_\star$ to the perturbed instance \( D_{j} \). We consider stochastic contexts, where each $u_t$ is the standard basis vector along coordinate $\ell_t$ sampled uniformly at random from $[d]$. Now, fix any contextual pricing policy $\cA$ for this this randomized environment. Our high-level intuition is that, for each coordinate $\ell \in [d]$, the sub-environment during the roughly $T/d$ rounds when $\ell_t = \ell$ mirrors that of \cref{lem:1d-lower}, and so we incur regret $\Omega(\sqrt{K_\star \cdot T/d})$ during such rounds. Summing over $\ell \in [d]$ then gives the lower bound.

To formalize this, note that each coordinate is sampled at least $T' = \lfloor T/2d \rfloor$ times under an event $\cE$ with probability at least $1-1/T$. This follows by a Chernoff bound a union bound over coordinates since $T \geq 8d\log(2d)$. Then, for each $\ell \in [d]$, there is a natural policy $\cA_\ell$ induced for the non-contextual setting of \cref{lem:1d-lower} with time horizon $T'$. To start, $\cA_\ell$ samples a valuation distribution $\tilde{Q}_{\ell'}$ uniformly at random from $\{Q_j\}_{j \in \{2, \dots, K_\star - 1\}}$ for each $\ell' \neq \ell$. Further, it instantiates a simulated copy of $\cA$ and a counter $\tau$, initialized at $1$, tracking the round of this simulation. Then, for round $t' = 1, \dots, T'$, $\cA_\ell$ performs the following:
\begin{enumerate}
    \item If $\tau > T$, play $p_{t'} = 1$ for remaining rounds and terminate.
    \item Otherwise, sample $\ell_\tau$ uniformly from $[d]$.
    \item Submit the associated context as $u_\tau$ to $\cA$ and receive suggested price $\tilde{p}_\tau$.
    \item If $\ell_\tau = \ell$, play price $p_{t'} = \tilde{p}_\tau$, submit the purchase feedback $\tilde{y}_\tau = y_{t'}$ to $\cA$, increment $\tau \gets \tau+1$, and continue to the next round.
    \item Otherwise, submit $\tilde{y}_\tau \sim \Ber(x)$ where $x$ is the demand of $\tilde{Q}_{\ell_\tau}$ at $\tilde{p}_\tau$, increment $\tau \gets \tau+1$, and return to Step 1.
\end{enumerate}
By design, if $\cA_\ell$ is run under the setting of \cref{lem:1d-lower}, its simulated copy of $\cA$ experiences feedback indistinguishable from that described in our setting above. Thus, writing $\cT_{\ell}$ for the rounds of our initial setting where $\ell_t = \ell$ and conditioning on $\cE$ (under which $|\cT_\ell| \geq T'$), we have
\begin{align*}
    \E\left[\,\sum_{t \in \cT_\ell} \gap_\star(p_t,u_t) \,\Biggm|\, \cE\right] &\geq \E[R_{\cA_\ell}(T') \mid \cE]\\
    &\geq \E[R_{\cA_\ell}(T')] - 1\\
    &= \Omega(\sqrt{K_\star T/d}),
\end{align*}
since $T' \geq K_\star^3$ and $T \geq 2d$. We then compute
\begin{align*}
    \E[R_\cA(T)] &\geq \E[R_\cA(T) \mid \cE] - 1\\
    &= \E\left[\,\sum_{\ell=1}^d \sum_{t \in \cT_\ell} \gap_\star(p_t,u_t) \,\Biggm|\, \cE\right] - 1\\
    &= \sum_{\ell=1}^d \E\left[\, \sum_{t \in \cT_\ell} \gap_\star(p_t,u_t) \,\Biggm|\, \cE\right] - 1\\
    &= \Omega\bigl(\sqrt{K_\star T d}\bigr),
\end{align*}
as desired.

\xhdr{Small $K_\star$.}
If $K_\star = 2$, a simpler one-dimensional construction suffices. We set $v_1 = 1/4$, $v_2 = 1/2$, $Q_0(v_1) = Q_0(v_2) = 1/2$, so that $\dem_{Q_0}(v_1) = 1$, $\dem_{Q_0}(v_1) = 1/2$, and $\rev_{Q_0}(v_1) = \rev_{Q_0}(v_2) = 1/4$. We then define $Q_\pm(v_1) = 1/2 \mp \eps$ and $Q_\pm(v_2) = 1/2 \pm \eps$, so that $\rev_{Q_\pm}(v_1) = 1/4$ and $\rev_{Q_\pm}(v_2) = 1/4 \pm \eps/2$. Moreover, taking $\eps = 1/\sqrt{T}$ (valid so long as $T \geq 4$, which we assumed), we can simply employ the standard $2$-armed bandits lower bound (e.g., using the same techniques as in \cref{prf:1d-lower}) to show that no algorithm can achieve regret $o(\sqrt{T})$ for both instances. Our argument above, tensoring one-dimensional instances and obtaining a contextual lower bound, still goes through, since $Q_+$ and $Q_-$ share the same set of probability weights, giving a lower bound of $\Omega(\sqrt{T d})$ for a worst-case instance. For $K_\star = 3$, we can easily tweak the $K_\star = 2$ construction to place negligible mass at $v_3 = 0$, and the lower bound still holds.
\qed

\subsection{Non-contextual Lower Bound (Proof of Lemma~\ref{lem:1d-lower})}
\label{prf:1d-lower}

We first recall some basic information theory. Write $\KL(p\,\|\,q) \defeq \E_p[\log(\dd p / \dd q)]$ for the Kullback-Leibler divergence between distributions $p$ and $q$ on the same domain $\cX$. When $p$ and $q$ are Bernoulli distributions with success probabilities $a,b \in [0,1]$, we write $\KL(a\,\|\,b) = \KL(p\,\|\,q)$.

\begin{fact}[Pinsker's inequality]
\label{fact:Pinsker}
For $p, q \in \Delta(\cX)$ and $M \geq 0$, we have
\begin{equation*}
    \sup_{f:\cX \to [0,M]} \E_{p}[f] - \E_{q}[f] \leq M\sqrt{\tfrac{1}{2}\KL(p\,\|\,q)}.
\end{equation*}
\end{fact}

\begin{fact}[Bernoulli KL bound]
\label{fact:bernoulli-KL-bd}
For $a,\delta \in [0,1]$ such that $a + \delta \in (0,1)$, we have
\begin{equation*}
    \KL(a\,\|\,a+\delta) \leq \frac{\delta^2}{(a + \delta)(1 - a - \delta)}.
\end{equation*}
\end{fact}

\begin{fact}[KL chain rule]
\label{fact:KL-chain-rule}
For distributions $p,q$ over sequences $X^n \!=\! (X_1, \dots, X_n) \!\in\! \cX^n$,
\begin{equation*}
    \KL(p\|q) = \sum_{i=1}^n \E_{X^n \sim p}\left[\KL\bigl(p(X_i \mid X_1, \dots, X_{i-1})\,\|\, q(X_i \mid X_1, \dots, X_{i-1})\bigr)\right],
\end{equation*}
where $p(X_i | X_1, \dots, X_{i-1})$ denotes the conditional distribution of $X_i$ under $p$ given $X_1, \dots, X_{i-1}$.
\end{fact}

Next, we observe that, if the buyer valuations are drawn from $Q_j$, the price $v_{j}$ achieves expected revenue at least $1/2 + \Omega(\eps)$ whereas all other $v_i \neq v_j$ achieve revenue $1/2$. Indeed, the demand function of $Q_j$ is identical to $Q_0$ at all $v_i \neq v_{j}$, and at $v_j$ we have
\begin{align*}
    \rev_{Q_j}(v_j) = v_j (\dem_{Q_0}(v_j) + \eps) = 1/2 + \eps v_j \geq 1/2 + \eps/2.
\end{align*}
Thus, $\gap_{Q_j}(v_i) \geq \eps/2 \cdot \mathds{1}\{i = j\}$. Also, for $i > 1$, \cref{eq:lb-demand} implies that $1/2 \leq \dem_{Q_0}(v_i) \leq 5/6$. So, imposing that our perturbation size $\eps$ is less than $1/12$, we have that
\begin{equation}
\label{eq:lb-KL-bd}
    \KL\bigl(\dem_{Q_0}(v_i)\,\|\,\dem_{Q_j}(v_i)\bigr) \leq \frac{\eps^2}{\tfrac{1}{2} \cdot \tfrac{1}{12}} = 24\eps^2.
\end{equation}
for all $i,j \in \{2, \dots, K_\star - 1\}$.

Now, let $J$ be drawn uniformly at random from $\{2, \dots, K_\star-1\}$, so that $Q_J$ coincides with the random instance from the lemma statement. Conditioned on $J$, let the valuations $V^T = (V_1,\dots,V_T)$ be drawn i.i.d.\ from $Q_J$, i.e., $V^T \sim Q_J^T$. Under this set up, $v_J$ is the unique optimal price, and playing $v_i$ for $i \neq J$ incurs regret $\Omega(\eps)$. We will be comparing to the alternative world where $V^T \sim Q_0^T$ and all arms have equal expected revenue of $1/2$. All expectations under this alternative will be clearly denoted as such.

Without loss of generality, we assume that the fixed pricing algorithm $\cA$ is deterministic (since we may condition on any internal randomness of $\cA$). Applying $\cA$ to instance $Q_J$ induces a random sequence of prices $P_1,\ldots,P_T$, where $P_t$ each is a function of the previous purchase decisions $Y_\tau = \mathds{1}\{V_\tau \geq P_\tau\}$ for $\tau \in [t-1]$. Without loss of generality, we may assume that each $P_t$ belongs to $\{v_2, \dots, v_{K_\star -1}\}$, since rounding up to the nearest element of this set can only increase expected revenue. Thus, defining the empirical frequencies
\begin{equation*}
    N_j \defeq \sum_{t=1}^T \mathds{1}\{P_t = v_j\},
\end{equation*}
we have $\sum_{j=2}^{K_\star -1} N_j = T$. Conditioned on $J$, the regret of $\mathcal{A}$ over the $T$ rounds is
\begin{align*}
    R_\cA(T) &= \sum_{t=1}^T \gap_{Q_J}(P_t)\\
    &\geq \frac{\eps}{2} \sum_{t=1}^T \mathds{1}\{P_t \neq v_J\}\\
    &\geq \frac{\eps}{2} \left(T - N_J\right),
\end{align*}
Thus, we have in expectation over $J$ and $V^T$ that
\begin{align}
\label{eq:key-regret-lb}
    \E_{J,V^T}[R_\cA(T)] \geq \frac{\eps}{2} \left(T - \E_{J,V^T}[N_J]\right).
\end{align}
We next control $\E_{J,V^T}[N_J]$ using the KL chain rule and Pinsker's inequality. Fixing $j \in \{2, \dots, K_\star - 1\}$, write $q_j(\cdot)$ and $q_0(\cdot)$ for the induced distributions on the entire purchase sequence $Y^T = (Y_1,\dots,Y_T)$ under $Q_j$ and $Q_0$, respectively. By the chain rule, we compute
\begin{align*}
    \KL(q_0 \,\|\, q_j) &= \sum_{t=1}^T \E_{Y^T \sim q_0}\left[\KL\bigl(q_0(Y_t \mid Y_1, \dots, Y_{t-1}) \,\|\, q_j(Y_t \mid Y_1, \dots, Y_{t-1})\bigr)\right]\\
    &= \sum_{t=1}^T \E_{Y^T \sim q_0}\left[\KL\bigl(q_0(Y_t \mid P_t) \,\|\, q_j(Y_t \mid P_t)\bigr) \cdot \mathds{1}\{P_t = v_j\}\right]\\
    &= \sum_{t=1}^T \E_{Y^T \sim q_0}\left[\KL\bigl(\dem_{Q_0}(v_j) \,\|\, \dem_{Q_j}(v_j)\bigr) \cdot \mathds{1}\{P_t = v_j\}\right]\\
    &\leq 24 \eps^2 \sum_{t=1}^T \E_{Y^T \sim q_0}\left[\mathds{1}\{P_t = v_j\}\right] \tag{Eq.\ \ref{eq:lb-KL-bd}}\\
    &= 24 \eps^2 \E_{V^T \sim Q_0^T}[N_j]
\end{align*}
In short, $q_0$ and $q_j$ are only distinguishable for rounds in which $\cA$ selects price $v_j$, and even in these rounds their divergence is bounded. Now, since $N_j$ is a deterministic function of $Y^T$ for fixed $\cA$,
\[
\mathbb{E}_{Y^T \sim q_j}[N_j]
\;\le\;
\mathbb{E}_{Y^T\sim q_0}[N_j] + \epsilon\,T\,\sqrt{24\,\mathbb{E}_{V\sim Q_0}[N_j]}.
\]
Taking an expectation over $J$ then gives
\begin{align*}
\E_{J,V^T}[N_J] &\leq \mathbb{E}_{J}\Bigl[\mathbb{E}_{V^T\sim Q_0^T}[\,N_J]\Bigr]
\;+\;\epsilon\,T\,\sqrt{24\,\mathbb{E}_{J}\Bigl[\mathbb{E}_{V^T\sim Q_0^T}[\,N_J]\Bigr]}.
\end{align*}
Since $\sum_{j=2}^{K_\star - 1} N_j = T$ and $J$ is uniform over $\{2,\dots,K_\star-1\}$, $K_\star \geq 4$, it follows that 
\begin{equation*}
    \mathbb{E}_{J}\Bigl[\mathbb{E}_{V^T\sim Q_0^T}[\,N_J]\Bigr] = \frac{1}{K_\star -2}\sum_{j=2}^{K_\star - 1} \mathbb{E}_{V^T\sim Q_0^T}[N_j] = \frac{T}{K_\star - 2} \leq \frac{T}{2}.
\end{equation*}
Plugging this into \eqref{eq:key-regret-lb} gives
\begin{align*}
\E_{J,V^T}[R_\cA(T)]  &\geq \frac{\eps}{2} \cdot \bigl(T - \E_{J,V^T}[N_J]\bigr)\\
&\geq \frac{\eps}{2} T \left(\frac{1}{2} - \eps \sqrt{24 \cdot \frac{T}{K_\star - 2}}\right)
\end{align*}
We now set $\eps = \frac{1}{3\sqrt{24}}\sqrt{(K_\star - 2)/T}$. Our construction required that $\eps \leq 1/(2K_\star - 2)$ and $\eps \leq 1/12$, which are satisfied under our assumption that $T \geq K_\star^3$. This yields a final lower bound of
\[
\mathbb{E}_{J,V^T}[R_\mathcal{A}(T)]
\geq \frac{\eps}{2} \cdot T \left( \frac{1}{2} - \eps\sqrt{24\frac{T}{K_\star-2}} \right) = \Omega\Bigl(\!\sqrt{K_\star T}\Bigr),
\]
as desired.
\qed

%% file: sections/C-zooming-proofs.tex
\section{Proofs for Section~\ref{sec:zooming}}
\label{app:zooming-proofs}

To fully specify the algorithm, we introduce the following notation:

\squishlist
    \item Since our analysis exclusively reasons about the true type distribution $D_\star$, we abbreviate $\dem = \dem_{\star}$, $\rev = \rev_{\star}$, $\gap = \gap_{\star}$, and $\br = \br_\star$.
    \item For each price $p \in [0,1]$ and round $t \in [T]$, we write
    \begin{itemize}
        \item $\cT_t(p) \defeq \{ \tau \in [t-1] : p_\tau = p \}$ for the set of previous rounds where $p$ was played,
        \item $n_t(p) \defeq |\cT_t(p)|$ for the count of these rounds,
        \item $\mu_t(p) \defeq \frac{1}{n_t(p)}\sum_{\tau \in \cT_t(p)} p y_t$ for the average revenue during these rounds,
        \item $V_t(p) \defeq \frac{1}{n_t(p)-1}\sum_{\tau \in \cT_t(p)} (py_t - \mu_t(p))^2$ for the sample variance, and
        \item $\sigma^2(p) \defeq p^2 \dem(p)(1-\dem(p))$ for the population variance (unknown to the seller).
    \end{itemize}
    When $n_t(p) = 0$, we set $\mu_t(p) = 0 = V_t(p) = 0$. When $n_t(p) = 1$, take $V_t(p) = \infty$.
    \item Defining confidence radius $r_t(p) \defeq \sqrt{\frac{10V_t(p) \log T}{n_t(p)}} + \frac{12 \log(T)}{n_t(p)-1}$ (taken as $+\infty$ if $n_t(p) \leq 1$), a variant of Bernstein's inequality shows that $|\mu_t(p) - \mathsf{rev}(p)| \leq r_t(p)$ w.h.p.\ (see \cref{lem:zooming-concentration}).
    \item Write $\UCB_t(p) \defeq \mu_t(p) + r_t(p)$, so that $\rev(p) \leq \UCB_t(p)$ w.h.p.
    \item We say a price $p$ is \emph{covered} by $q \in S$ if $p \in [q, q + r_t(q)]$ and $q$ is the largest active price no greater than $p$, i.e., $q = \max\{q' \in S: q' \leq p\}$.
    One-sided Lipschitzness of the revenue function (Lemma~\ref{lem:one-sided-lipschitzness}) implies that $\rev(p) - \rev(q) \leq r_t(q)$ w.h.p.
    \item Define the index of a price $q \in S$ as $\idx_t(q) \defeq \UCB_t (q) + r_t(q)$. Each price $p$ covered by some $q \in S$ satisfies $\rev(p) \leq \idx_t(q)$.
\squishend

\begin{algorithm2e}[!t]
\SetAlgoNoEnd
    \caption{\Zooming: Variance-Aware Zooming for Non-Contextual Pricing}
    \label{alg:zooming}
    \textbf{Initialize}: active price set $S \leftarrow \{\frac{2^i}{T}: i = 0,1, \dots, \lfloor \log_2 T \rfloor \} \cup \{1\}$\;
    \For{each round $t \in [T]$}{
        play $p_t \in \argmax_{q \in S} \idx_t(q)$\;
        observe $y_t = \mathds{1} \{ \theta_t \leq p_t\}$ and update $n_{t+1}(p_t)$\;
        \If{a price $p > 1/T$ becomes uncovered}{
        $q \gets \min\{q' \in S: q' > p_t\}$\label{step:zooming-added-arm}\;
        $S \gets S \cup \{(p_t + q)/2\}$}
    }
\end{algorithm2e}

\subsection{Main Regret Bound for \Zooming (Proof of Theorem~\ref{thm:zooming})}

As mentioned, the lower bound follows by that in \cref{thm:pops} when $d=1$. For the upper bound, we prove a generic regret bound in \cref{prf:regret-zooming} depending on the variance-aware zooming dimension.

\begin{lemma}
\label{lem:regret-zooming}
For $c > 0$, \Zooming achieves regret $\widetilde{O}\left(c^{1/(2+z)}\,T^{1 - 1/(2+z)}\right)$, where $z = \ZoomDimV(c)$.
\end{lemma}

Next, we prove the zooming dimension bounds claimed in \cref{sec:zooming}.

\begin{lemma}
\label{lem:zoom-dim-bound}
We have $\ZoomDimV(10K_\star) = 0$ and $\ZoomDimV(10) \leq \ZoomDim(10) \leq 1$.
\end{lemma}
\begin{proof}
For each $\delta > 0$ and type $i \in [K_\star]$, let $X_\delta^{\smash{(i)}}$ denote the set of activated arms $p$ with $\gap_\star(p) \leq \delta$ that lie in the interval $(\theta^{(i-1)},\theta^{(i)}]$, to the left of type $i$. Since revenue is linearly increasing within each such interval, with slope $d_i = \dem(\theta^{(i)})$, the gap condition requires that each $p \in X_\delta^{\smash{(i)}}$ also satisfies $p \geq \theta^{(i)} - \delta d_i^{-1}$. Moreover, for $p \in X_\delta^{\smash{(i)}}$, we have $\sigma^2(p) \leq d_i$. Thus, we obtain
\begin{equation*}
    N_\mathrm{v}(X_\delta,\delta/20) \leq \sum_{i\in [K_\star]} N_\mathrm{v}(X_\delta^{(i)},\delta/20) \leq \sum_{i \in [K_\star]} 10 d_i^{-1} \cdot d_i \leq 20{K_\star},
\end{equation*}
implying the first bound. For the second, we note that $N(X_\delta,\delta/20) \leq N([0,1],\delta/20) \leq 20\delta^{-1}$.
\end{proof}

Combining the two lemmas gives the theorem. Indeed, the $\sqrt{K_\star T}$ bound follows by Lemma~\ref{lem:regret-zooming} with $c = 10K_\star$ and $z = 0$, using Lemma~\ref{lem:zoom-dim-bound}. Similarly, the $T^{2/3}$ bound follows by taking $c=10$ and $z = 1$.
\qed

\subsection{Base Regret Bound for \Zooming (Proof of Lemma~\ref{lem:regret-zooming})}
\label{prf:regret-zooming}

We begin with a few helper lemmas. Throughout, we assume that $T \geq 3$ (otherwise the regret bound holds trivially). Our proofs mirror those of similar lemmas in \cite{slivkins2019introduction}, with small adjustments to handle the variance-adjusted confidence radii and the dyadic price selection rule.

\begin{lemma}[Concentration]
\label{lem:zooming-concentration}
Write $\cE_{\mathrm{clean}}$ for the event that 
\begin{equation*}
    |\mu_t(p) - \rev(p)| \leq r_t(p) \leq \sqrt{\frac{10\,\sigma^2(p) \log T}{n_t(p)}} + \frac{126\log(T)}{n_t(p) - 1}
\end{equation*}
for all $t \in [T]$ and for all $p \in [0,1]$. Then $\PP(\cE_{\mathrm{clean}}) \geq 1 - 8T^{-2}$.
\end{lemma}
\begin{proof}
For fixed $p \in [0,1]$ and $t \in [T]$, Theorems 10 and 11 of \citet{maurer2009empirical} imply that
\begin{align*}
    |\mu_t(p) - \rev(p)| &\leq r_t(p)\\
    \bigl|\sigma_t(p) - \sqrt{V_t(p)}\bigr| &\leq \sqrt{\frac{11\log T}{n_t(p)-1}}
\end{align*}
with probability $1-8T^{-5}$. We note that similar bounds appear in \cite{audibert2009exploration}, which inspired our adjustments to the confidence intervals. Under this event, we further bound
\begin{align*}
    r_t(p) &= \sqrt{\frac{10V_n(p) \log T}{n_t(p)}} + \frac{12 \log T}{n_t(p)-1}\\
    &\leq \sqrt{\frac{10\sigma^2(p) \log T}{n_t(p)}} + \sqrt{\frac{10\log T}{n_t(p)}}\sqrt{\frac{11\log T}{n_t(p)-1}} + \frac{12 \log T}{n_t(p)-1}\\
    &\leq \sqrt{\frac{10\sigma^2(p) \log T}{n_t(p)}} +  \frac{126 \log T}{n_t(p)-1}.
\end{align*}
Taking a union bound over $t$, the above must hold for all $t \in [T]$ with probability at least $1-8T^{-4}$. Now, the same Chernoff bound argument used in Claim 4.13 of \citet{slivkins2019introduction} implies that $|\mu_t(p) - \rev(p)| \leq r_t(p)$ for all $p \in [0,1]$ and $t \in [T]$ with probability at least $1-8T^{-2}$. One technical observation is that Claim 4.13 requires that the set of arms every played by the algorithm is finite. This holds for \Zooming due to our dyadic arm activation rule.
\end{proof}

\begin{lemma}[Covering invariant]
\label{lem:zooming-invariant}
At the beginning of each round, every price $p \geq 1/T$ is covered by some active arm.
\end{lemma}
\begin{proof}
At round 1, $r_t(q) = \infty$ for all active arms $q \in S$, and so all arms larger than $1/T$ are covered by our choice of $S$. Now, suppose that the lemma holds up to round $t$, and that playing $p_t$ causes a price $p \in R$ to become uncovered. Then, by the definition of covering, we must have
\begin{equation*}
    p_t \leq p_t + r_{t+1}(p_t) < p < q \leq p_t + r_t(p_t),
\end{equation*}
where $q$ is the nearest active price to the right of $p_t$, selected at Step~\ref{step:zooming-added-arm}. First, verify that the added price, $p' \defeq (p_t + q)/2$, is less than $p$. Since $r_{t+1}(p_t)$ must be less than one, we must have $n_{t+1}(p) - 1 = n_t(p) \geq 12$. Thus, we can bound
\begin{align*}
    V_{t+1}(p) &= \frac{p^2}{n_{t+1}(p)(n_{t+1}(p)-1)}\sum_{\tau_1 \in \cT_{t+1}(p)}\sum_{\tau_2 \in \cT_{t+1}(p)}(y_{\tau_1} - y_{\tau_2})^2\\
    &\geq \frac{p^2}{(n_{t}(p)+1)n_{t}(p)}\sum_{\tau_1 \in \cT_t(p)}\sum_{\tau_2 \in \cT_t(p)}(y_{\tau_1} - y_{\tau_2})^2\\
    &= \frac{n_{t}(p)-1}{n_{t}(p)+1} V_t(p)\\
    &\geq \frac{11}{13} V_t(p).
\end{align*}
Consequently, one can show that
\begin{align*}
    r_{t+1}(p_t) = \sqrt{\frac{10V_{t+1}(p) \log T}{n_{t}(p)+1}} + \frac{12 \log(T)}{n_{t}(p)} \geq \frac{3}{4} \sqrt{\frac{10V_{t}(p) \log T}{n_{t}(p)}} + \frac{12 \log(T)}{n_{t}(p)-1} = \frac{3}{4}r_t(p_t).
\end{align*}
Combining, we find that $p' = (p_t + q)/2 \leq p_t + \frac{1}{2} r_t(p_t) < p_t + r_{t+1}(p_t) < p$, as desired. Moreover, $p'$ could not have already been active at round $t$; otherwise, $p_t$ would not have been covering $p$. Finally, once $p'$ is added, it covers $p$ since $r_{t+1}(p') = \infty$.
\end{proof}

\begin{lemma}[Gap bound]
\label{lem:gap-bound}
Condition on $\cE_\mathrm{clean}$. Then $\gap(p) \leq 5r_t(p)$ and $n_t(p) \leq \frac{252 \sigma^2(p) \log T}{\gap(p)^2}+ \frac{504 \log T}{\gap(p)}$ for all $p \in [0,1]$ and $t \in [T]$.
\end{lemma}
\begin{proof}
Write $\hat{p} = \max\{\br,\frac{1}{T}\}$, and fix any price $p \in [0,1]$. Consider some round $t$ at which $p$ is played. By Lemma~\ref{lem:zooming-invariant}, we know that $\hat{p}$ was covered at the beginning of round $t$ by some price $q \in S$, and that $p$ had a higher index than $q$. Hence,
\begin{align*}
    \idx_t(p) &\geq \idx_t(q) \tag{since $p$ was played}\\
    &= \UCB_t(q) + r_t(q) \tag{by definition of index}\\
    &\geq \mu(q) + r_t(q) \tag{concentration guarantee}\\
    &\geq \rev(\hat{p}). \tag{$\hat{p}$ covered by $q$}
\end{align*}
Moreover, $\idx_t(p) \leq \mu_t(p) + 2r_t(p) \leq \rev (p) + 3 r_t(p)$. Thus, 
\begin{equation*}
    \gap (p) = \rev (p_\star) - \rev(p) \leq \frac{1}{T} + \rev(\hat{p}) - \rev(p) \leq \frac{1}{T} + 3r_t(p).
\end{equation*}
Now, if $n_t(p) \leq 12$, $\gap(p) \leq 1 < r_{t+1}(p)$. Otherwise, the bound above implies that $\gap(p) \leq 5 r_{t+1}(p)$. Since $r_t(p)$ only changes when $p$ is played and $\gap(p) \leq 5r_t(p)$ when $t=1$, this guarantee holds for all $t$. For the other bound, we apply concentration to obtain
\begin{align*}
    \gap(p) \leq 5 \left(\sqrt{\frac{10\,\sigma^2(p) \log T}{n_t(p)}} + \frac{126\log(T)}{n_t(p) - 1}\right).
\end{align*}
Rearranging and solving the quadratic inequality in $n_t(p)$ gives the stated result.
\end{proof}

Compared to the standard gap bound lemma for zooming (see, e.g., Lemma 4.14 of \citealp{slivkins2019introduction}), Lemma~\ref{lem:gap-bound} is adapted to the variance of each price $p$. We can now prove the regret bound.

\begin{lemma}[Active arm separation]
\label{lem:arm-separation}
Conditioned on $\cE_\mathrm{clean}$, consider any three consecutive active arms $x < y < z$ which did not belong to $S$ at initialization. Then $z - x > \frac{1}{10}  \min\bigl\{\gap(x), \gap(y) \bigr\}$.
\end{lemma}
\begin{proof}
If $y$ was activated before $z$, then $z$ must have been added as the midpoint of active arms $y$ and $y + 2(z-y) = 2z-y$ at round $\tau_z$, when $y$ must have not covered $2z-y$. Thus, by Lemma~\ref{lem:gap-bound}, we would have $2(z-x) > 2(z-y) = (2z-y)-y > \frac{1}{5}\gap(y)$. On the other hand, if $y$ was activated after $z$, then it must have been added as the midpoint of $x$ and $z$ at round $\tau_y$, when $x$ must have not covered $z$. Again, by Lemma~\ref{lem:gap-bound}, this would imply $z-x > \frac{1}{5}\Delta(y)$.
\end{proof}

\begin{altproof}{Proof of Lemma~\ref{lem:regret-zooming}}
We freely condition on $\cE_\mathrm{clean}$, since the complement has negligible probability $O(T^{-2})$. For each $\delta > 0$, let $Y_\delta \subseteq X_{2\delta}$ denote the set of activated prices $p$ with $\gap(p) \in [\delta,2\delta)$. In what follows, we say that two prices are adjacent if they are neighboring within $Y_\delta$. Note that at most $O(\log T)$ of the prices in $Y_\delta$ were activated at initialization. Consider the set $Y^0_\delta$ which, for each such price, contains this price and up to two neighboring prices, such that the remaining prices $Y_\delta \setminus Y_\delta^0$ can be split into triples of neighboring prices. We then decompose
\begin{equation*}
    Y_\delta \setminus Y_\delta^0 = \{ p_{1,1} < p_{1,2} < p_{1,3} < p_{2,1} < p_{2,2} < p_{2,3} < \dots < p_{n,1} < p_{n,2} < p_{n,3} \},
\end{equation*}
where $p_{i,1}$, $p_{i,2}$, $p_{i,3}$ are neighboring for each $i \in [n]$. By Lemma~\ref{lem:arm-separation}, we have $p_{i,3} - p_{i,1} > \delta/10$ for all $i \in [n]$, and so $p_{i,k} - p_{j,k} > \delta/10$ for all $k\in [3]$ whenever $i < j-1$. Thus, we can partition $Y_\delta \setminus Y_\delta^0$ into at most $6$ packings, each of which has separation at least $\delta/10$. Of course any $(\delta/10)$-packing of $Y_\delta \subseteq X_{2\delta}$ is contained within a $(\delta/5)$-cover of $X_{2\delta}$. Consequently, we have
\begin{equation*}
    \sum_{p \in Y_\delta} \sigma^2(p) \leq 6 N_\mathrm{var}(X_{2\delta},\delta/5) + O(\log T) \leq 6 c \delta^{-z} + O(\log T).
\end{equation*}
Noting that a $(\delta/10)$-packing within $[0,1]$ can have cardinality at most $10\delta^{-1}$, we further bound $|Y_\delta| = O(\log T + \delta^{-1})$. Thus, by Lemma~\ref{lem:gap-bound}, the regret incurred due to posting prices in $Y_\delta$ is at most
\begin{align*}
    2\delta \cdot \sum_{p \in Y_\delta} O\left(\sigma^2(p) \log(T) \delta^{-2} + \log(T) \delta^{-1}\right) &= O\left( \log(T) \delta^{-1} \sum_{p \in Y_\delta}\sigma^2(p)  + \log(T) |Y_\delta|\right)\\
    &= O\Bigl( \log(T) \left(c\delta^{-1-z} + \log(T)\delta^{-1} \right)  + \log(T) \left(\log T + \delta^{-1}\right)\Bigr)\\
    &= O\bigl( c\log(T)\delta^{-1-z} + \log^2(T) \delta^{-1}\bigr)
\end{align*}
Now we sum over $\delta = 1/2, 1/4, \dots, \alpha$, where $\alpha$ will be tuned later, giving a total regret bound of
\begin{align*}
    \sum_{j=1}^{\log(1/\alpha)} \left[c\log(T)2^{j(1+z)} + \log^2(T) 2^j \right] + \alpha T &= O\left(c \log(T) \alpha^{-(1+z)} + \log^2(T) \alpha^{-1} + \alpha T\right). 
\end{align*}
Taking $\alpha = (c\log(T)/T)^{1/(2+z)}$, we obtain the desired bound of $\widetilde{O}\left(c^{1/(2+z)} T^{1-1/(2+z)}\right)$.
\end{altproof}

%% file: sections/D-extras-proofs.tex
\section{Proofs for Section~\ref{sec:extras}}\label{app:extras}

We first recall some results from Vapnik–Chervonenkis (VC) theory. Given distributions $D,D'$ over a finite domain $\cX$, define the total variation distance $\|D - D'\|_\TV \defeq \sup_{A \subseteq \cX}|D(A) - D'(A)|$.

\begin{lemma}[Section 28.1 of \citealp{shalev2014understanding}]
\label{lem:vc}
Fix a finite set $\cX$, a function family $\cF \subseteq \{0,1\}^\cX$, a distribution $D \!\in\! \Delta(\cX)$, and $\delta \!\in\! (0,1)$. Then, for $X_1, \dots, X_n$ sampled i.i.d.\ from $D$, we have
\begin{equation*}
    \sup_{f \in \cF} \,\left|\E_{x \sim D}[f(x)] - \frac{1}{n}\sum_{i=1}^n f(X_i)\right| = O\left(\sqrt{\frac{V \log(n/V) + \log(1/\delta)}{n}}\right)
\end{equation*}
with probability at least $1 - \delta$, where $V$ is the VC dimension of $\cF$.
\end{lemma}

\begin{lemma}[Theorem 9.3 of \citealp{shalev2014understanding}]
\label{lem:vc-cor}
The family $\cF = \{0,1\}^\cX$ has VC dimension $|\cX|$, and the family of linear thresholds over $\R^d$ has VC dimension $d+1$. The former result implies that, under the setting above, we have
\begin{equation*}
    \|D - \hat{D}_n\|_\TV = O\left(\sqrt{\frac{|\cX| \log(n/|\cX|) + \log(1/\delta)}{n}}\right)
\end{equation*}
with probability at least $1 - \delta$.
\end{lemma}

We also recall Bernstein's inequality for the case of i.i.d.\ Bernoulli random variables.

\begin{lemma}[Theorem 2.10 of \citealp{boucheron2013concentration}]
\label{lem:bernsteins}
For i.i.d.\ $X_1, \dots, X_n \sim \Ber(p)$ and $\delta > 0$,
\begin{equation*}
    \sum_{i=1}^n X_i \leq pn + \sqrt{2np\log(1/\delta)} + \log(1/\delta)/3
\end{equation*}
with probability at least $1-\delta$.
\end{lemma}

We now turn to the main proofs.

\subsection{Observed Type Identifiers (Proof of Theorem~\ref{thm:unknloc_unknmas})}
\label{prf:unknloc_unknmas}

To state our result for the first setting, we introduce an $\eps$-ball performance metric for contextual search. This is a slight strengthening of the standard $\eps$-ball metric $\sum_{t=1}^T \mathds{1}\{ |p_t - v_t| > \eps \}$.

\begin{definition}[Strong $\eps$-ball regret]
Let $\cA$ be a contextual pricing policy which, at round $t \in [T]$ with context $u_t$, outputs a price $p_t \in [0,1]$ along with a  confidence width $w_t \in [0,1]$. For $\eps \in (0,1]$, we say that $\cA$ achieves strong $\eps$-ball regret $R(T)$ for contextual search if, when $K_\star = 1$ and $D_\star = \delta_{\theta_\star}$, we have $|p_t - v_t| \leq w_t$ for each round $t$ and $\sum_{t=1}^T \mathds{1}\{w_t > \eps\} \leq R(T)$, where $v_t = \langle u_t, \theta_\star \rangle$.
\end{definition}

That is, $\cA$ produces $\eps$-accurate estimates for the true values, outside of up to $R(T)$ rounds for exploration, \emph{and} it can identify when these estimates are accurate. In practice, this tends to require that $\cA$ maintain a confidence set around $\theta_\star$ whose width, when projected onto the current context, is greater than $\eps$ for at most $R(T)$ rounds. Fortunately, there are existing efficient algorithms which achieve low $\eps$-ball regret.

\begin{lemma}[\citealp{lobel18contextual}]
\label{lem:eps-ball-regret}
For $\eps \in (0,1]$, there exists a contextual search algorithm $\ProjectedVolume(\eps)$, based on the ellipsoid method, with strong $\eps$-ball regret $O(d \log (d/\eps))$ and running time $\poly(d,1/\eps)$ per round.\footnote{Although \cite{lobel18contextual} state a slightly weaker guarantee, instead bounding $\sum_{t=1}^T \mathds{1}\{|p_t - v_t| > \eps\} = O(d\log(d/\eps))$, this strengthened result is immediate from their proof.}
\end{lemma}

We now present our algorithm (Algorithm~\ref{alg:ident}), which uses \ProjectedVolume as a subroutine.

\paragraph{Overview of Algorithm~\ref{alg:ident}}
We maintain a set $\cI$ of observed types, initially empty, and an accuracy $\eps$ (tuned to minimize regret). We will initialize an independent copy $\cA_i$ of \ProjectedVolume for each $i$ added to $\cI$. Since these copies are simulated, we are free to query the price $\mathsf{price}(\cA_i,u_t)$ and confidence width $\mathsf{width}(\cA_i,u_t)$ for a context $u_t$ without updating $\cA_i$. Moreover, for each $i \in \cI$, the algorithm maintains a frequency count $n_t(i)$, recording the number of rounds which we have followed the recommended price of $\cA_i$, along with an exploration count $m_t(i)$, recording the number of rounds which we have played the price of $\cA_i$ due to its lack of confidence along the current context. At each round $t$, we perform the following:
\begin{itemize}
    \item If there is an observed type $i \in \cI$ such that $\mathsf{width}(\cA_i,u_t) > \eps$ and that its number of exploration plays $m_t(i)$ is below a threshold of $\eps T/K_\star$, the algorithm plays $\mathsf{price}(\cA_i,u_t)$, observes the outcome and the realized type $z_t$, and updates $\cA_i$ if $z_t = i$. In addition, we increment $n_t(i)$ and $m_t(i)$.
    \item Otherwise, it defines active set $\cS=\{i\in\cI: \mathsf{width}(\cA_i,u_t)\le \eps\}$, computes for each $i \in S$ the score
    \begin{equation*}
    F(i)=\sum_{j\in\cS} \frac{n_t(j)}{t-1}\,\mathds{1}\Bigl\{\mathsf{price}(\cA_j,u_t)\ge \mathsf{price}(\cA_i,u_t)\Bigr\},
    \end{equation*}
    and plays
    $p_t=\max\{\mathsf{price}(\cA_{i^*},u_t) - \eps,0\}$ where $i^*\in\argmax_{i\in\cS} \bigl\{F(i)\cdot \mathsf{price}(\cA_i,u_t)\bigr\}$. It then observes $y_t$ and $z_t$ and updates the frequency count $n_{t+1}(z_t)$. Here, $F$ is an estimate for the demand at the price suggested by $\cA_i$, and so $i^\star$ is an estimate for the revenue maximizing type. We pull back the price recommended by $\cA_{i^\star}$ by $\eps$ to avoid issues due to estimation error. 
\end{itemize}

\begin{algorithm2e}[t]
\SetAlgoNoEnd
    \caption{Contextual Pricing with Ex-Post Type Identification}
    \label{alg:ident}
    \textbf{initialize}: observed types $\mathcal{I} = \varnothing$, $\eps = \sqrt{d \log(T)/T}$\;
    \For{each round $t \in [T]$}{
        observe context $u_t$\;
        \If{exists $i \in \cI$ such that $\mathsf{width}(\cA_i, u_t) > \eps$ and $m_t(i) < \eps T$ \label{step:ident-check}}{
                play $p_t = \mathsf{price}(\cA_i, u_t)$ and observe $y_t$\;
                observe type $z_t \in [K_\star]$\;
                update algorithm $\cA_i$ with $y_t$ if $z_t = i$\;
                increment $m_t(i)$ by 1\;
        } 
        \Else{
            let $\cS = \{ i \in \mathcal{I}: \mathsf{width}(\cA_i, u_t) \leq \eps \}$\;
            define $F (i) = \sum_{j \in \cS} \frac{n_t(j)}{t-1} \cdot \mathds{1} \{ \mathsf{price}(\cA_j, u_t) \geq \mathsf{price}(\cA_i, u_t) \}$ for each $i \in \cS$\;
            set $i^* = \argmax_{i \in \cS} F(i) \cdot \mathsf{price}(\cA_i, u_t)$\;
            play $p_t = \max\{\mathsf{price}(\cA_{i^*},u_t) - \eps,0\}$ and observe $y_t$\;
            observe type $z_t \in [K_\star]$\;
        }
        increment $n_t(z_t)$ by 1\;
        \If{$z_t \not\in \cI$}{
            initialize copy $\cA_{z_t}$ of $\ProjectedVolume(\eps)$ and set $\cI \gets \cI \cup \{ z_t \}$\;
        }
    }
\end{algorithm2e}

\xhdr{Bounding exploration regret.} Write $\cT_1$ for the set of exploration rounds. By design, an exploration round is one in which some type \(i\) is used with \(\mathsf{width}(\cA_i,u_t)>\eps\) and exploration counter satisfying $m_t(i)<\eps T$. Trivially, $|\cT_1| \le \eps T$, so we can incur regret at most $\eps T = \widetilde{O}(\sqrt{dT})$ during exploration.\smallskip

\xhdr{Bounding mass of types which saturate exploration threshold.} Next, consider any type \(i\) that has been explored sufficiently so that \(m_T(i)=\eps T\) after time $T$; denote by \(\cS'\) the set of all such types. We will show that $D_\star$ plays small mass on $\cS'$. Fix $i \in \cS'$ and write $\cT_{1,i}$ for the exploration rounds where we follow $\cA_i$. Conditioned on $\cT_{1,i}$, we note that $X_t = \mathds{1}\{z_t = i\}$, $t \in \cT_{1,i}$, are i.i.d.\ Bernoulli random variables with \(\Pr(X_t=1)=D_\star(i)\). Defining
\[
S_i=\sum_{\tau \in \cT_{1,i}}X_\tau,
\]
our guarantee for \ProjectedVolume (Lemma~\ref{lem:eps-ball-regret}) and the width condition for exploration imply that $S_i = O(d\log(d/\eps))$. On the other hand, by Bernstein's inequality (Lemma~\ref{lem:bernsteins}), we have
\[
S_i\ge m_T(i)D_\star(i)-\sqrt{2\,m_T(i)D_\star(i)\,\log(T)}-\frac{1}{3}\log (T).
\]
with probability at least $1 - 1/T$. Since \(m_T(i)=\eps T = \sqrt{d T \log(T)}\), the dominant term is \(m_T(i)D_\star(i)\) for $T$ greater than a sufficiently large constant. Thus, we deduce that
\[
\eps T\,D_\star(i)\le O\left(d\log\frac{d}{\eps}\right),
\]
or equivalently,
\[
D_\star(i)\le O\left(\frac{ d\log(d/\eps)}{\eps T}\right).
\]
Summing over all types in \(\cS'\) (of which there are at most \(K_\star\)) and taking a union bound, the total mass in \(\cS'\) is at most
\begin{equation}
\label{eq:rare-type-bd}
D_\star(\cS') = \sum_{i\in \cS'}D_\star(i)\le O\left(\frac{K_\star\, d\log(d/\eps)}{\eps T}\right) = \widetilde{O}\left(K_\star \sqrt{d/T}\right)
\end{equation}
with probability at least $1 - K_\star/T$. We condition on this bound holding for the remainder of the proof, since doing so contributes a negligible $K_\star$ to the regret. We also condition on the event that, for each round $t \in [T]$ the empirical frequencies of (all) types deviate from their true masses by at most $O\bigl(\sqrt{K_\star\log(T)/t}\bigr)$ in total variation. This is permissible by Lemma~\ref{lem:vc-cor} and a union bound over rounds.

\xhdr{Bounding exploitation regret.} 
Fix an exploitation round $t$, and recall the set of accurately estimated types 
\[
\cS=\{i\in\cI:\mathsf{width}(\cA_i,u_t)\le \eps\}.
\]
Write $v_1, \dots, v_{K_\star} \in [0,1]$ for the true values at round $t$. By our construction and the $\eps$-ball guarantee for \ProjectedVolume, $\cA_i$ returns a price that is an \(\eps\)-accurate estimate of $v_i$, for each $i \in \cS$. Moreover, by our analysis above, the mass on types outside of $\cS$ is quite small. We thus bound
\begin{align*}
    \dem_\star(p_t,u_t) &= \sum_{i=1}^{K_\star} D_\star(i) \mathds{1}\{ v_i \geq p_t \}\\
    &\geq \sum_{i \in \cS} D_\star(i) \mathds{1}\{ v_i \geq p_t \} -\widetilde{O}\left(K_\star \sqrt{d/T}\right) \tag{Eq.\ \eqref{eq:rare-type-bd}}\\
    &\geq \sum_{i \in \cS} \frac{n_t(i)}{t-1} \mathds{1}\{ v_i \geq p_t \} - \widetilde{O}\left(K_\star \sqrt{d/T}\right) - O\bigl(\sqrt{K_\star\log(T)/t}\bigr) \tag{TV bound}\\
    &\geq \sum_{i \in \cS} \frac{n_t(i)}{t-1} \mathds{1}\{ v_i \geq \price(\cA_{i^\star},u_t) - \eps \} - \widetilde{O}\left(K_\star \sqrt{d\log(T)/t}\right) \tag{choice of $p_t$}\\
    &\geq \sum_{i \in \cS} \frac{n_t(i)}{t-1} \mathds{1}\{ \price(\cA_{i},u_t) \geq \price(\cA_{i^\star},u_t) \} - \widetilde{O}\left(K_\star \sqrt{d\log(T)/t}\right) \tag{$i \in \cS$}\\
    &= F(i_\star) - \widetilde{O}\left(K_\star \sqrt{d\log(T)/t}\right). \tag{choice of $F$}
\end{align*}
Consequently, we bound $\rev_\star(p_t,u_t) + \widetilde{O}\left(K_\star \sqrt{d\log(T)/t}\right)$ from below by
\begin{align*}
    p_t F(i^*) &\geq \price(\cA_{i*},u_t) F(i^*) - \eps\\
    &= \max_{j \in S} \price(\cA_{j},u_t) F(j) - \eps \tag{choice of $i^*$}\\
    &= \max_{p \in [0,1]} p \sum_{i \in \cS} \frac{n_t(i)}{t-1} \mathds{1}\{ \price(\cA_{i},u_t) \geq p \} - \eps \tag{rev.\ maximized at jump}\\
    &\geq \max_{p \in [0,1]} p \sum_{i \in \cS} D_\star(i) \mathds{1}\{ \price(\cA_{i},u_t) \geq p \} - \eps - \widetilde{O}\left(K_\star \sqrt{d\log(T)/t}\right) \tag{TV bound}\\
    &\geq \max_{p \in [0,1]} p \sum_{i \in \cS} D_\star(i) \mathds{1}\{ v_i \geq p + \eps \} - \eps - \widetilde{O}\left(K_\star \sqrt{d\log(T)/t}\right) \tag{$i \in S$}\\
    &\geq \max_{p \in [0,1]} p \sum_{i =1}^{K_\star} D_\star(i) \mathds{1}\{ v_i \geq p + \eps \} - \eps - \widetilde{O}\left(K_\star \sqrt{d\log(T)/t}\right) \tag{Eq.\ \eqref{eq:rare-type-bd}}\\
    &= \max_{p \in [0,1]} p \,\dem_\star(p+\eps,u_t) - \eps - \widetilde{O}\left(K_\star \sqrt{d\log(T)/t}\right) \\
    &= \max_{p \in [0,1]} \rev_\star(p+\eps,u_t) - 2\eps - \widetilde{O}\left(K_\star \sqrt{d\log(T)/t}\right) \\
    &= \max_{p \in [0,1]} \rev_\star(p,u_t) - 3\eps - \widetilde{O}\left(K_\star \sqrt{d\log(T)/t}\right).
\end{align*}
All together, we see that playing $p_t$ incurs regret at most $\widetilde{O}\left(K_\star \sqrt{d\log(T)/t}\right)$. Summing over exploitation rounds and adding the exploration regret gives a total bound of
\begin{equation*}
    R(T) = \sum_{t=1}^T\widetilde{O}\left(K_\star \sqrt{d\log(T)/t}\right) + \widetilde{O}(\sqrt{dT}) = \widetilde{O}(K_\star \sqrt{dT}),
\end{equation*}
as desired.\qed

\subsection{Observed Type Vectors (Proof of Theorem~\ref{thm:knloc_unknmas})}
\label{sec:knloc_unknmas}

We first show that $\rev_{\hat{D}_\tau}$ concentrates tightly around $\rev_\star$, using a simple VC bound.

\begin{lemma}
\label{lem:demand-concentration}
Fix $D \in \Delta_K(\Theta)$ and let $\hat{D}_t$ be the empirical measure of $t$ i.i.d.\ samples from $D$. We then have
\begin{equation*}
    \sup_{p \in [0,1], u \in \cU} |\rev_D(p,u) - \rev_{\hat{D}_t}(p,u)| = O\left(\sqrt{\frac{\min\{K,d\} \log(t) + \log(1/\delta)}{t}}\right)
\end{equation*}
with probability at least $1 - \delta$.
\end{lemma}
\begin{proof}
We compute
\begin{align*}
    \sup_{p \in [0,1], u \in \cU} |\rev_D(p,u) - \rev_{\hat{D}_t}(p,u)| &\leq \sup_{p \in [0,1], u \in \cU} |\dem_D(p,u) - \dem_{\hat{D}_t}(p,u)|\\
    &= \sup_{f \in \cF} \left|\E_{\theta \sim D}[f(\theta)] - \E_{\theta \sim \hat{D}_t}[f(\theta)]\right|,
\end{align*}
where $\cF$ is the space of linear threshold functions $f_{p,u} : \supp(D) \to \{0,1\}$ given by $f_{p,u}(\theta) = \mathds{1}\{\langle u,\theta \rangle \geq p \}$. The result then follows by Lemma~\ref{lem:vc-cor}.
\end{proof}

Now, our best response policy ensures that, at each round $t > 1$, we have
\begin{align*}
    \rev_\star(p_t,u_t) &= \rev_{\hat{D}_{t-1}}(p_t, u_t) + \rev_{\star}(p_t,u_t) - \rev_{\hat{D}_{t-1}}(p_t, u_t)\\
    &\geq \rev_{\hat{D}_{t-1}}(\br_\star(u_t), u_t) - \sup_{p \in [0,1], u \in \cU} |\rev_{\star}(p,u) - \rev_{\hat{D}_{t-1}}(p, u)|\\
    &\geq \rev_{\star}(\br_\star(u_t), u_t) - 2\sup_{p \in [0,1], u \in \cU} |\rev_{\star}(p,u) - \rev_{\hat{D}_{t-1}}(p, u)|.
\end{align*}
Consequently, regret is at most
\begin{align*}
    R(T) &\leq 1 + 2\sum_{t=1}^{T-1} \sup_{p \in [0,1], u \in \cU} |\rev_{\star}(p,u) - \rev_{\hat{D}_{t}}(p, u)|\\
    &= O\left(\sqrt{\min\{K_\star,d\}}\sum_{t=1}^{T-1} \sqrt{\frac{ \log (t)}{t}}\right) \tag{Lemma~\ref{lem:demand-concentration} with $\delta = t^{-2}$}\\
    &= \widetilde{O}(\sqrt{\min\{K_\star,d\} T}),
\end{align*}
as desired.\qed

%% file: arxiv_main.bbl
\begin{thebibliography}{}

\bibitem[Amin et~al., 2014]{amin2014repeated}
Amin, K., Rostamizadeh, A., and Syed, U. (2014).
\newblock Repeated contextual auctions with strategic buyers.
\newblock In {\em Advances in Neural Information Processing Systems}.

\bibitem[Audibert et~al., 2009]{audibert2009exploration}
Audibert, J.-Y., Munos, R., and Szepesv{\'a}ri, C. (2009).
\newblock Exploration--exploitation tradeoff using variance estimates in multi-armed bandits.
\newblock {\em Theoretical Computer Science}, 410(19):1876--1902.

\bibitem[Bacchiocchi et~al., 2025]{bacchiocchi25piecewise}
Bacchiocchi, F., Castiglioni, M., Marchesi, A., and Gatti, N. (2025).
\newblock Regret minimization for piecewise linear rewards: Contracts, auctions, and beyond.
\newblock In {\em ACM Conference on Economics and Computation (EC)}.

\bibitem[Boucheron et~al., 2013]{boucheron2013concentration}
Boucheron, S., Lugosi, G., and Massart, P. (2013).
\newblock {\em Concentration Inequalities: A Nonasymptotic Theory of Independence}.
\newblock Oxford University Press.

\bibitem[Cesa{-}Bianchi et~al., 2019]{cesabianchi19pricing}
Cesa{-}Bianchi, N., Cesari, T., and Perchet, V. (2019).
\newblock Dynamic pricing with finitely many unknown valuations.
\newblock In {\em Algorithmic Learning Theory}.

\bibitem[Cohen et~al., 2020]{cohen20pricing}
Cohen, M.~C., Lobel, I., and Paes~Leme, R. (2020).
\newblock Feature-based dynamic pricing.
\newblock {\em Management Science}, 66(11):4921--4943.

\bibitem[Fan et~al., 2024]{fan2024policy}
Fan, J., Guo, Y., and Yu, M. (2024).
\newblock Policy optimization using semiparametric models for dynamic pricing.
\newblock {\em Journal of the American Statistical Association}, 119(545):552--564.

\bibitem[Foster et~al., 2021a]{foster21instance}
Foster, D., Rakhlin, A., Simchi-Levi, D., and Xu, Y. (2021a).
\newblock Instance-dependent complexity of contextual bandits and reinforcement learning: A disagreement-based perspective.
\newblock In {\em Conference on Learning Theory (COLT)}.

\bibitem[Foster et~al., 2021b]{foster21statistical}
Foster, D.~J., Kakade, S.~M., Qian, J., and Rakhlin, A. (2021b).
\newblock The statistical complexity of interactive decision making.
\newblock {\em arXiv preprint arXiv:2112.13487}.

\bibitem[Hanneke, 2014]{hanneke2014theory}
Hanneke, S. (2014).
\newblock Theory of disagreement-based active learning.
\newblock {\em Foundations and Trends{\textregistered} in Machine Learning}, 7(2-3):131--309.

\bibitem[Ho et~al., 2014]{ho2014adaptive}
Ho, C.-J., Slivkins, A., and Vaughan, J.~W. (2014).
\newblock Adaptive contract design for crowdsourcing markets: Bandit algorithms for repeated principal-agent problems.
\newblock In {\em ACM Conference on Economics and Computation (EC)}.

\bibitem[Javanmard, 2017]{javanmard2017perishability}
Javanmard, A. (2017).
\newblock Perishability of data: dynamic pricing under varying-coefficient models.
\newblock {\em Journal of Machine Learning Research}, 18(53):1--31.

\bibitem[Javanmard and Nazerzadeh, 2019]{javanmard2019dynamic}
Javanmard, A. and Nazerzadeh, H. (2019).
\newblock Dynamic pricing in high-dimensions.
\newblock {\em Journal of Machine Learning Research}, 20(9):1--49.

\bibitem[Kleinberg et~al., 2008]{slivkins2008}
Kleinberg, R., Slivkins, A., and Upfal, E. (2008).
\newblock Multi-armed bandits in metric spaces.
\newblock In {\em ACM Symposium on Theory of Computing (STOC)}.

\bibitem[Kleinberg and Leighton, 2003]{kleinberg03demand}
Kleinberg, R.~D. and Leighton, F.~T. (2003).
\newblock The value of knowing a demand curve: Bounds on regret for online posted-price auctions.
\newblock In {\em Symposium on Foundations of Computer Science (FOCS)}.

\bibitem[Krishnamurthy et~al., 2020]{krishnamurthy2020contextual}
Krishnamurthy, A., Langford, J., Slivkins, A., and Zhang, C. (2020).
\newblock Contextual bandits with continuous actions: Smoothing, zooming, and adapting.
\newblock {\em Journal of Machine Learning Research}, 21(137):1--45.

\bibitem[Krishnamurthy et~al., 2023]{krishnamurthy23contextual}
Krishnamurthy, A., Lykouris, T., Podimata, C., and Schapire, R. (2023).
\newblock Contextual search in the presence of adversarial corruptions.
\newblock {\em Operations Research}, 71(4):1120--1135.

\bibitem[Liu et~al., 2021]{liu21contextual}
Liu, A., Leme, R.~P., and Schneider, J. (2021).
\newblock Optimal contextual pricing and extensions.
\newblock In {\em Symposium on Discrete Algorithms (SODA)}.

\bibitem[Lobel et~al., 2018]{lobel18contextual}
Lobel, I., {Paes Leme}, R., and Vladu, A. (2018).
\newblock Multidimensional binary search for contextual decision-making.
\newblock {\em Operations Research}, 66(5):1346--1361.

\bibitem[Luo et~al., 2024]{luo2024distribution}
Luo, Y., Sun, W.~W., and Liu, Y. (2024).
\newblock Distribution-free contextual dynamic pricing.
\newblock {\em Mathematics of Operations Research}, 49(1):599--618.

\bibitem[Maurer and Pontil, 2009]{maurer2009empirical}
Maurer, A. and Pontil, M. (2009).
\newblock Empirical {B}ernstein bounds and sample variance penalization.
\newblock In {\em Conference on Learning Theory (COLT)}.

\bibitem[{Paes Leme} et~al., 2022]{leme22corruptcontextual}
{Paes Leme}, R., Podimata, C., and Schneider, J. (2022).
\newblock Corruption-robust contextual search through density updates.
\newblock In {\em Conference on Learning Theory (COLT)}.

\bibitem[{Paes Leme} and Schneider, 2022]{leme22contextual}
{Paes Leme}, R. and Schneider, J. (2022).
\newblock Contextual search via intrinsic volumes.
\newblock {\em SIAM Journal on Computing}, 51(4):1096--1125.

\bibitem[Paes~Leme et~al., 2023]{leme2023pricing}
Paes~Leme, R., Sivan, B., Teng, Y., and Worah, P. (2023).
\newblock Pricing query complexity of revenue maximization.
\newblock In {\em ACM-SIAM Symposium on Discrete Algorithms (SODA)}, pages 399--415.

\bibitem[Podimata and Slivkins, 2021]{podimata2021adaptive}
Podimata, C. and Slivkins, A. (2021).
\newblock Adaptive discretization for adversarial {L}ipschitz bandits.
\newblock In {\em Conference on Learning Theory (COLT)}.

\bibitem[Shalev-Shwartz and Ben-David, 2014]{shalev2014understanding}
Shalev-Shwartz, S. and Ben-David, S. (2014).
\newblock {\em Understanding machine learning: From theory to algorithms}.
\newblock Cambridge University Press.

\bibitem[Slivkins, 2011]{slivkins2011contextual}
Slivkins, A. (2011).
\newblock Contextual bandits with similarity information.
\newblock In {\em Conference On Learning Theory (COLT)}.

\bibitem[Slivkins et~al., 2019]{slivkins2019introduction}
Slivkins, A. et~al. (2019).
\newblock Introduction to multi-armed bandits.
\newblock {\em Foundations and Trends{\textregistered} in Machine Learning}, 12(1-2):1--286.

\bibitem[Xu and Wang, 2022]{xu2022towards}
Xu, J. and Wang, Y.-X. (2022).
\newblock Towards agnostic feature-based dynamic pricing: Linear policies vs linear valuation with unknown noise.
\newblock In {\em International Conference on Artificial Intelligence and Statistics (AISTATS)}.

\bibitem[Zhang, 2022]{zhang2022thompson}
Zhang, T. (2022).
\newblock Feel-good {T}hompson sampling for contextual bandits and reinforcement learning.
\newblock {\em SIAM Journal on Mathematics of Data Science}, 4(2):834--857.

\bibitem[Zuo, 2024]{zuo2024corruption}
Zuo, S. (2024).
\newblock Corruption-robust {L}ipschitz contextual search.
\newblock In {\em International Conference on Algorithmic Learning Theory (ALT)}.

\end{thebibliography}
